\documentclass[10pt,journal]{IEEEtran}
\usepackage[numbers,sort&compress]{natbib}
\pdfoutput=1
\usepackage{array,color,url}
\usepackage[citecolor=black,linkcolor=black,urlcolor=black]{hyperref}
\usepackage{amsfonts,amsmath,amsthm,amssymb}
\usepackage{subfigure,graphicx}
\usepackage{epstopdf}
\usepackage{setspace}
\usepackage[pagewise]{lineno}
\usepackage[ruled]{algorithm2e}
\theoremstyle{definition}
\newtheorem{theorem}{Theorem}
\newtheorem{proposition}{Proposition}
\newtheorem{lemma}{Lemma}

\newtheorem{definition}{Definition}

\newcommand{\RNum}[1]{\expandafter{\romannumeral #1\relax}}
\hypersetup{hidelinks}

\begin{document}
\begin{spacing}{1.0}	
\title{Latent Dirichlet Allocation Model Training with Differential Privacy}
\author{Fangyuan~Zhao, Xuebin~Ren, Shusen~Yang, Qing~Han, Peng~Zhao, and Xinyu~Yang
}

\maketitle

\begin{abstract}
Latent Dirichlet Allocation (LDA) is a popular topic modeling technique for hidden semantic discovery of text data and serves as a fundamental tool for text analysis in various applications. 		
However, the LDA model as well as the training process of LDA may expose the text information in the training data, thus bringing significant privacy concerns. To address the privacy issue in LDA, we systematically investigate the privacy protection of the main-stream LDA training algorithm based on Collapsed Gibbs Sampling (CGS) and propose several differentially private LDA algorithms for typical training scenarios. In particular, we present the first theoretical analysis on the inherent differential privacy guarantee of CGS based LDA training and further propose a centralized privacy-preserving algorithm (HDP-LDA) that can prevent data inference from the intermediate statistics in the CGS training. Also, we propose a locally private LDA training algorithm (\textsf{LP-LDA}) on crowdsourced data to provide local differential privacy for individual data contributors. Furthermore, we extend \textsf{LP-LDA} to an online version as \textsf{OLP-LDA} to achieve LDA training on locally private mini-batches in a streaming setting. Extensive analysis and experiment results validate both the effectiveness and efficiency of our proposed privacy-preserving LDA training algorithms. 

\end{abstract}

\section{Introduction}
\IEEEPARstart{L}{atent} Dirichlet Allocation (LDA)~\cite{blei2003latent} is a basic building block widely used in many machine learning (ML) applications. 
In essence, LDA works by mapping the high dimensional word space to a low dimensional topic space while preserving the implicit probabilistic relationship. Therefore, LDA is often used as a dimension reduction tool for extracting main features from massive text datasets, thereby simplifying the subsequent text processing tasks like classification and similarity judgment.
With such great benefits, a series of large-scale LDA platforms have been developed in the era of big data, including light LDA~\cite{yuan2015lightlda} for Microsoft, Peacock~\cite{wang2014towards} and LDA*~\cite{yut2017lda} for Tencent, and Yahoo!LDA~\cite{smola2010architecture} for Yahoo!.	

As shown in Fig.~\ref{LDAapplications}, with a broad spectrum of application fields, LDA may be trained on information-sensitive datasets from both large enterprises/organizations and massive crowdsourcing users.
For example, patients' electronic health records can be fed in an LDA model to help doctors diagnose possible diseases~\cite{hoogendoorn2016utilizing}. Social media profiles can be utilized with LDA models to perform fine-grained community discovery~\cite{zhang2007lda}.
However, similar to other ML techniques, the LDA model improves its utility by constantly exploring the raw data, thus inevitably memorizing some knowledge about the dataset. By utilizing this characteristic, several attack models have been proposed to distill private information of the training data from machine learning models. For example, membership inference attack~\cite{shokri2017membership} has been proved to be able to extract the membership of training samples. Model inversion attack~\cite{fredrikson2014privacy} can be launched to recover the training data by observing the model predictions. As a typical ML model, LDA may also suffer from these attacks and cause severe privacy risks.

\begin{figure}[t]
\centering
\includegraphics[width=0.95\linewidth]{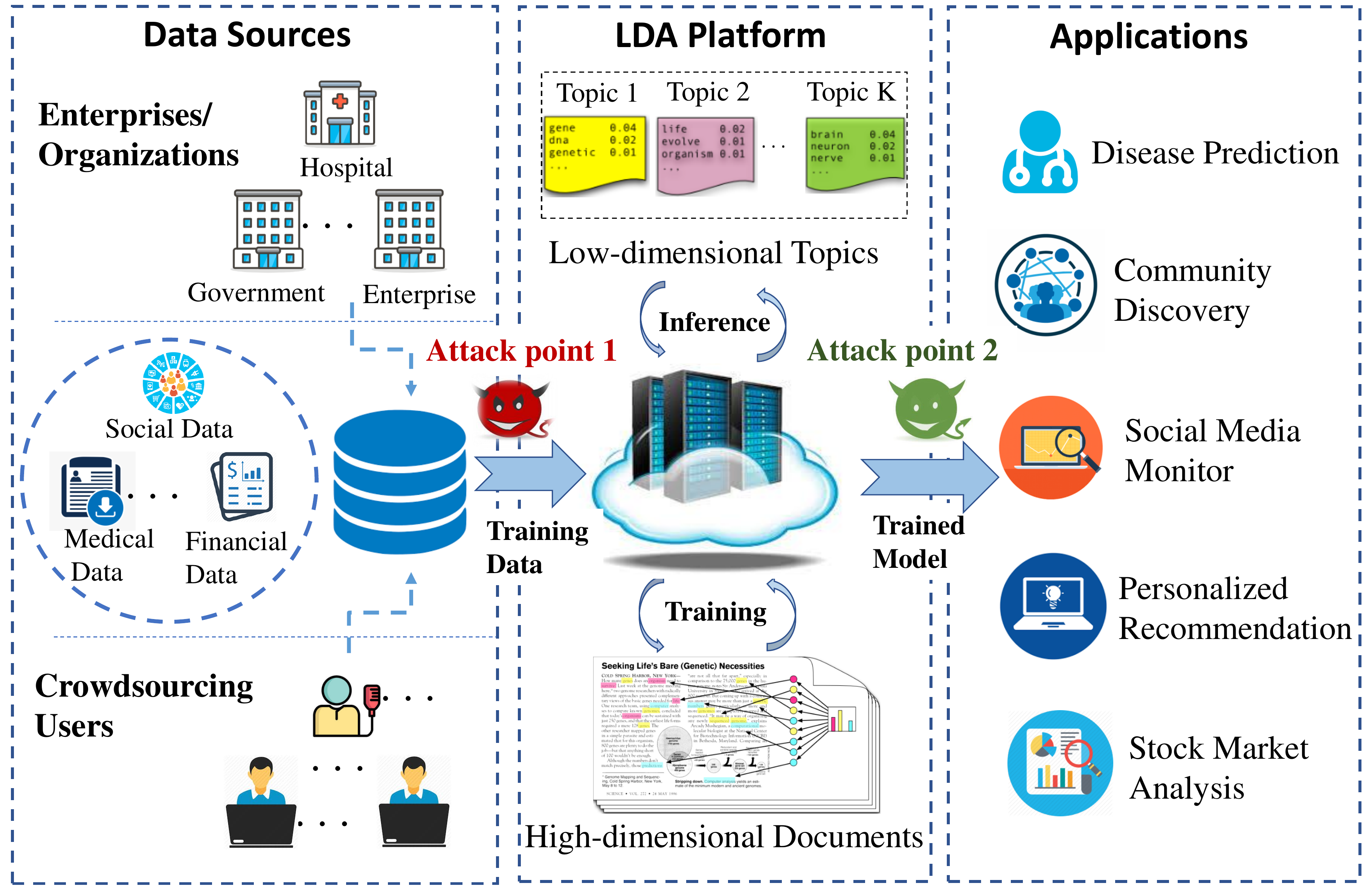}
\caption{Application Scenarios of LDA Model Training}
\label{LDAapplications}
\end{figure}

Differential privacy (DP)~\cite{dwork2006calibrating}, as a rigorous paradigm for privacy preservation, provides not only a mathematical framework for quantifying the privacy risks of existing algorithms, but also an efficient guidance for customizing privacy-preserving algorithms. As a result, it has become the de-facto standard of privacy preservation and has been adopted in a broad wide of applications like data publication~\cite{ren2018textsf,li2019impact,zhu2014correlated} and machine learning~\cite{chaudhuri2011differentially,abadi2016deep}. Similarly, DP for LDA has also attracted lots of research interests~\cite{park2016variational,zhu2016privacy,wang2020federated}. For example, Park et al.~\cite{park2016variational} provided DP for variational Bayesian algorithm based LDA training.
Zhu et. al.~\cite{zhu2016privacy} introduced DP into the collapsed Gibbs sampling(CGS)~\cite{griffiths2002gibbs} training process, by adding noise to the word counts statistics in the last iteration of the sampling process. Wang et. al.~\cite{wang2020federated} proposed a local private solution to LDA training in a federated setting. 	

However, there are several limitations regarding the existing studies. On one hand, few works consider strong adversaries with full knowledge of the training mechanism and access to intermediate statistics for CGS-based LDA training (a common adversary model similar to~\cite{abadi2016deep}). In particular, in the training process, both the word-count information and the sampled topics for each word can reveal the information about specific training samples. Nevertheless, the existing work~\cite{zhu2016privacy} simply focused on protecting the former but neglected the latter.
On the other hand, these methods mainly prevent third-party adversaries from stealing the trained model (e.g., Attack Point 2 in Fig.~\ref{LDAapplications}). They implicitly assume the LDA model is trained on centralized datasets by a trustworthy server. However, the central server may act as an ``honest but curious'' adversary and steal training data silently (e.g., Attack Point 1 in Fig.~\ref{LDAapplications}). Then, users may be reluctant to share their data directly.   
Furthermore, most of the existing works focus on batched LDA for static text data, thus cannot efficiently cope with many practical scenarios where training data comes in a streaming fashion.

To address the issues, in this paper, we not only investigate to provide more comprehensive protection for the whole training process by utilizing the inherent privacy of the CGS algorithm on centralized datasets,
but also present a solution to LDA training with local differential privacy (LDP). In addition, we further extend the LDP solution to an efficient online LDA scenario. Our contributions are summarized as follows:

\begin{enumerate}

\item We present the first study on the inherent privacy of CGS-based LDA by discovering the consistency between the topic sampling of CGS and the exponential mechanism of DP. Based on the study, we propose \textsf{HDP-LDA}, the first privacy-preserving LDA algorithm to protect the whole training process of CGS-based LDA on centralized datasets.

\item We propose \textsf{LP-LDA}, a novel privacy-preserving mechanism that supports training an LDA model on crowd-sourced datasets. \textsf{LP-LDA} can provide local differential privacy for individual data contributors.
\item We propose \textsf{OLP-LDA}, a privacy-preserving online LDA algorithm for crowd-sourced data streams. In particular, we first provide a baseline online LDA framework \textsf{O-LDA}, and then propose to utilize prior knowledge to refine model accuracy on locally private data  streams.
\item We conduct extensive experiments on several real-world datasets to validate the effectiveness of the proposed algorithms. Experiment results show that the proposed algorithms can achieve much higher model utility while providing sufficient privacy guarantees.
\end{enumerate}

This paper is an extension of our preliminary work~\cite{DBLP:conf/ijcai/ZhaoRYY19}, which focuses on the privacy preservation of batch LDA training on static datasets.
Compared with \cite{DBLP:conf/ijcai/ZhaoRYY19}, more extensive analysis and experiments have been added to present a comprehensive study. Besides, the current paper further proposes
a novel online algorithm \textsf{OLP-LDA} to achieve efficient local differential privacy for LDA training on crowd-sourced data streams. In \textsf{OLP-LDA}, we first propose an online LDA training framework \textsf{O-LDA}, and then present a novel Bayesian denoising mechanism to enhance the utility of \textsf{O-LDA} on  privacy-preserved data batches by leveraging  prior knowledge. Finally, we perform extensive experiments to validate the performance of \textsf{OLP-LDA}.

The remainder of this paper is organized as follows. We discuss the related studies in Section~\ref{sec:related work}. We review the background of LDA and differential privacy in Section~\ref{sec:preliminaries}. We present the intrinsic privacy study of CGS based LDA in Section~\ref{sec:CGS-LDA}. Then, we give the LDP solutions for the batch and online  modes in Sections~\ref{sec:LP-LDA} and~\ref{sec:OLP-LDA}, respectively. 
The experiments and simulation results are described and explained in Section~\ref{sec:experiment}. Finally, we conclude this paper in Section~\ref{sec:conclusion}.

\section{Related work}\label{sec:related work}
\subsection{Machine Learning with Differential Privacy}

Many studies~\cite{chaudhuri2009privacy,chaudhuri2012near,jagannathan2009practical,xu2019ganobfuscator,huang2019dp,zhang2016dynamic} have adopted DP in the privacy preservation of ML models. The basic idea is to perturb the ML models in different parts. 
\textit{Output perturbation} on the final model result is the most straightforward solution. However, many ML models incur unbounded  sensitivities, which makes it difficult to implement. 
To mitigate this issue, \textit{sample-and-aggregate}~\cite{nissim2007smooth,dwork2009differential,czerniak2003application} framework is proposed to first train on disjointedly sample partitions and then aggregate the trained results with DP. Still, this technique applies to ML on relatively small datasets.
\textit{Objective perturbation} randomizes the cost function of ML models~\cite{chaudhuri2011differentially,bassily2014private}. 
\textit{Intermediate perturbation} aims to randomize the intermediate parameters during iterative training, which is effective for deep learning~\cite{park2016variational,shokri2015privacy,abadi2016deep}. Recently, \textit{input perturbation} that trains ML models on the perturbed datasets has attracted extensive research interests. A relevant notion in DP is called \textit{local differential privacy}~\cite{erlingsson2014rappor,sun2019relationship}, which shows that meaningful statistics could be obtained from massive randomized crowdsourced data. 
Both \textit{input perturbation} and \textit{local differential privacy} aim to eliminate the assumption for trustworthy servers, thus providing a stronger privacy guarantee.

\subsection{LDA Training with Differential Privacy}
Similar to other ML models, LDA can also achieve DP protection by the aforementioned strategies except for objective perturbation, since it has no explicit cost function. To the best of our knowledge, most of the current works achieve privacy protection by perturbing the intermediate parameters during the training process. For instance, Park et al.~\cite{park2016variational} proposed to obtain DP guarantee for LDA models by perturbing the expected sufficient statistics in each iteration of the variational Bayesian method, which is a parameter estimation algorithm for LDA.

Zhu et. al.~\cite{zhu2016privacy} presented a differentially private LDA algorithm by perturbing the sampling distribution in the last iteration of the CGS process~\cite{griffiths2002gibbs}, which is another typical training algorithm for LDA. Decarolis et. al.~\cite{decarolis2018end} decomposed the LDA training into a workflow based on the spectral algorithm (a parameter estimation algorithm for LDA) and then perturbed those intermediate statistics located in the cut of the workflow. There is other efficient differentially private parameter estimation algorithm using Bayesian inference with a Gibbs sampler~\cite{bernstein2018differentially}. However, the Gibbs sampler in it only considers a single conjugate structure and cannot be directly extended to LDA model training, which involves a coupled conjugate structure.
Besides, the above DP methods \cite{zhu2016privacy, griffiths2002gibbs, decarolis2018end} cannot defend against the untrustworthy data curators by design. Wang et. al.~\cite{wang2020federated} presented a locally private LDA training algorithm for the federated setting, in which the data are perturbed before uploading in each iteration. However, it is not a general method for traditional batch or online LDA learning. In this paper, we aim to present DP solutions to both batch and online LDA training which can defend against the adversaries with full knowledge of the training process, even the untrustworthy data curators.

\subsection{Intrinsic Privacy of Randomized Algorithm}

Several recent studies \cite{wang2015privacy,zhang2017privbayes,dimitrakakis2014robust,zhang2016differential} have begun to look into the intrinsic privacy guarantee provided by the randomized algorithms. For example, Wang et al.~\cite{wang2015privacy} and Dimitrakakis et al. \cite{dimitrakakis2014robust}
first demonstrated that, when meeting certain conditions of the loss function, the posterior Bayesian sampling and the stochastic gradient Markov Chain Monte Carlo (MCMC) techniques could possess some inherent privacy guarantee without the introduction of extra noise. Foulds et al.~\cite{foulds2016theory} further extended this conclusion to the general MCMC methods. By utilizing the inherent randomness of MCMC, they achieved a certain level of privacy protection equivalent to that assured by a Laplace mechanism. Minami et al.~\cite{minami2016differential} then relaxed the required conditions for the loss function in~\cite{wang2015privacy}. Nevertheless, these works investigated the inherent privacy guarantee of randomized algorithms in theory on the basis of some ideal assumptions about the parameters, such as the bounds of their sensitivities are known. However, in the LDA model training, it is often infeasible to compute or bound the parameters. In such a case, accurately measuring the inherent privacy guarantee remains a great challenge.

\section{Preliminaries}\label{sec:preliminaries}

\subsection{Collapsed Gibbs Sampling based LDA}
\subsubsection{LDA Generative Model}
LDA is a generative model, which describes the hidden semantic architecture of document corpus generation. In the view of LDA, each document $d_m$ containing $N_m$ words in the corpus (or text dataset\footnote{For simplicity, we use both terms  interchangeably.}) $D$, is a mixture of $K$ different topics and can be represented by a $K$-dimensional "document-topic" distribution $\theta_m$.  Each topic $k$ is characterized by a mixture of $V$ words, represented by a $V$-dimensional "topic-word" distribution $\phi_k$.

As shown in Fig.~\ref{Fig:LDAgraphmodel}, LDA defines the generative process of a given corpus as follows:

\begin{enumerate}
\item For each topic $k$, draw a ``topic-word" distribution $\phi_k\sim Dir(\beta)$ over all $V$ words, where $\beta$ is the hyperparameter describing the prior observation for the ``topic-word'' count.
\item For each document $d_m$, draw a ``document-topic" distribution $\theta_m\sim Dir(\alpha)$ over all $K$ topics, where $\alpha$ is the hyperparameter describing the prior observation for the ``document-topic'' count.
\item For each word $w_i$ in $d_m$, $1\leq i\leq N_m$, sample a topic $k\sim\theta_m$ and a word $t\sim\phi_k$.
\end{enumerate}

\begin{figure}[t]
\centering
\includegraphics[width=0.35\textwidth]{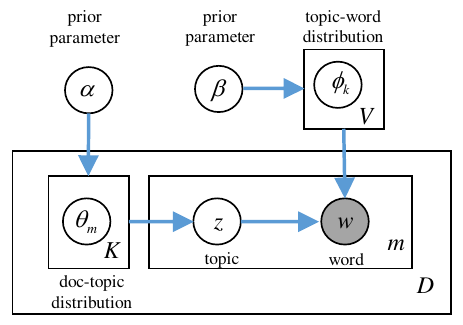}
\caption{LDA Graph Model}
\label{Fig:LDAgraphmodel}
\end{figure}

\subsubsection{Collapsed Gibbs Sampling}

The objective of LDA training is to learn the topic-word distribution $\mathbf{\phi}_k$ for each topic $k$, which can be used to infer the document-topic distribution $\mathbf{\theta}_m$ for any unseen document $d_m$.  Collapsed Gibbs Sampling (CGS) is the most popular training algorithm for LDA. As a special MCMC method, CGS works by generating topic samples alternatively for all the words in $D$, and then conducting Bayesian estimation for the topic-word distribution based on the generated topic samples. Three main procedures of CGS are summarized as follows:
\begin{itemize}
\item \textbf{Initialization.} In the beginning, each word $w\in D$ is randomly assigned with a topic $k\in \mathcal{K}$, and the word-count information $n_m^k$ and $n_k^t$ is counted\footnote{For simplicity, $n_m^k$ is called document-topic count and $n_k^t$ is called topic-word count in this paper.}.  $n_k^t$ denotes the number of times that word $t$ has been assigned with the topic $k$, and $n_m^k$ refers to the number of times that topic $k$ has been assigned to a word of the document $d_m$.
\item \textbf{Burn-in.} In each iteration, the topic assignment for each word $w\in D$ is updated alternatively by sampling from a multinomial distribution $\mathbf{P}=[p_1,...,p_k,...,p_K]$. Each component of $\mathbf{P}$ can be computed by
\begin{equation}\label{Equ:sampling equation}
p_k\propto\frac{n_{k}^{t}+\beta}{\sum_{t=1}^{V}(n_{k}^{t}+\beta)}\cdot\frac{n_{m }^{k}+\alpha}{\sum_{k=1}^{K}(n_{m}^{k}+\alpha)}
\end{equation}
where $p_k$ refers to the probability that topic $k$ is sampled. The word-count information $n_m^k$ and $n_k^t$ is updated in each sampling. After the given $T$ iterations, the burn-in process stops and the topic samples $\mathbf{z}$ can be obtained.
\item \textbf{Estimation.} The topic-word distribution $\phi_k$ for each topic $k$ is estimated based on the topic samples $\mathbf{z}$ and the words $\mathbf{w}\in D$. In particular, each component of $\phi_k$ can be estimated by
\begin{equation}\label{Equ: trained topic word distribution}
\mathbb{E}[\phi_k^t|\mathbf{z},\mathbf{w}]=\frac{n_{k}^{t}+\beta}{\sum_{t=1}^{V}(n_{k}^{t}+\beta)},
\end{equation}
where $\phi_k^t$ refers to the probability that word $t$ is generated by topic $k$ (corresponds to the assumption of LDA that words are generated by topics).
\end{itemize}
The detailed algorithm of CGS can be referred to in~\cite{heinrich2005parameter}.

\subsection{Differential Privacy and Exponential Mechanism} \label{sec:DP and EM}

Differential privacy (DP) provides a rigorous framework to quantify the privacy guarantee by analyzing the statistical difference between the algorithm outputs on neighboring datasets.

\begin{definition}
(\textbf{Differential Privacy}~\cite{dwork2006calibrating}) A randomized mechanism $\mathcal{M}:D\rightarrow Y$ satisfies $\epsilon\text{-differential privacy}$ if for any neighboring datasets $D, D'$ that differ by one record (i.e., $|D\oplus D'|=1$) and any output $S\subseteq Y$, there is
\begin{equation*}
\mathrm{Pr}[\mathcal{M}(D)\in S]\leq \mathrm{exp}(\epsilon) \cdot \mathrm{Pr}[\mathcal{M}(D')\in S],
\end{equation*}
where $\mathrm{Pr}[\cdot]$ is the probability and $\epsilon$ is the privacy level of $\mathcal{M}$.
\end{definition}

Exponential mechanism is a fundamental technique to achieve $\epsilon$-DP for the situations where a query requires an approximately "best" answer returned privately without any perturbation. The core idea of the exponential mechanism is to return an answer sampled from the answer set based on a certain distribution.
\begin{definition}\label{defination: exponential mechanism}
(\textbf{Exponential Mechanism}~\cite{dwork2006calibrating}) A mechanism $\mathcal{M}_E (x, u, \mathcal{R}): D\rightarrow R$ satisfies $\epsilon$-DP if $\mathcal{M}_E (x, u, \mathcal{R})$ outputs an element $r\in R$ with probability $p_r$ satisfies that
\begin{equation*}
p_r \propto \mathrm{exp}({\frac{\epsilon }{2\Delta u}}u(x,r))
\end{equation*}
where $u(x,r)$ is the utility function and $\Delta u$ is its sensitivity.
\end{definition}

\subsection{Local Differential Privacy}

As a variant of DP, local differential privacy (LDP) describes a new privacy paradigm that any two different inputs will be mapped to the same value with similar probability.

\begin{definition}
(\textbf{Local Differential Privacy}~\cite{dwork2014algorithmic}) A randomized function $f$ satisfies $\epsilon$-local differential privacy if and only if for any two input tuples $t$ and $t'$ in the domain of $f$, and for any output $t^*$ in the range of $f$, there is:
\begin{equation*}
\mathrm{Pr}[f(t)=t^*]\leq \mathrm{exp}(\epsilon) \cdot \mathrm{Pr}[f(t')=t^*],
\end{equation*}
where $\mathrm{Pr}[\cdot]$ is the probability and $\epsilon$ is the local differential privacy parameter.
\end{definition}

LDP can be implemented via randomized response mechanism and well adapted to the crowd-sourcing scenario where the central server may not be trustworthy to end users who prefer to protect their own data individually.

\section{HDP-LDA: A Hybrid Private LDA Training Algorithm}\label{sec:CGS-LDA}
In this section, we first point out the limitations of existing methods in protecting the intermediate statistics in LDA training. Then we present a systematical study on the inherent DP guarantee of CGS-based LDA training on centralized datasets. Finally, we propose a hybrid privacy-preserving algorithm \textsf{HDP-LDA}, which can protect all the intermediate statistics of the whole CGS-based LDA training.

\subsection{Limitations of the Existing Methods}\label{Limitations of the Existing method}
A direct method to achieve DP in the CGS-based LDA algorithm is to add noise to the inner statistics~\cite{foulds2016theory}\cite{zhu2016privacy}, e.g., word counts $n_k^t$ and $n_m^k$, based on which the sampling probability would be computed to perform topic sampling. In~\cite{foulds2016theory}, the authors provide a general method to achieve DP in Gibbs Sampling, i.e., adding Laplace noise to the sufficient statistics, $n_k^t$ and $n_m^k$ in LDA at the beginning of the Gibbs Sampling process. \cite{zhu2016privacy} proposes to add Laplace noise to $n_k^t$ and $n_m^k$ in the final iteration. Actually, both methods cannot really protect the training process against strong adversaries with full knowledge of the training mechanism and access to intermediate statistics in LDA. The reasons are as follows:
\begin{itemize}
\item \textbf{Insufficient protection on word-counts.} In some training scenarios, e.g., in the distributed training of CGS-based LDA~\cite{newman2009distributed}, the topic-word counts $n_k^t$ would be released frequently to synchronize the training progress among all parties participating in training. In such a scenario, simply adding noise in the first or final iteration cannot sufficiently prevent  privacy leakage in the word-count information\footnote{For simplicity, we also say "protect the word-count information".} released in the training process.
\item \textbf{No protection on the sampled topics.} Before each topic sampling, the CGS algorithm would first access to the word $w$ to be sampled on and then compute the sampling distribution based on the word-count information $n_k^t$ and $n_m^k$ according to Equation~(\ref{Equ:sampling equation}).
Therefore, the sampling process is not a post-process after sanitizing the word-counts since it depends not only on the sanitized word-counts but also on the word $w$ in the raw dataset. That is to say, the sampled topics would cause additional privacy cost. In some practical scenarios~\cite{wang2014towards}, the topic assignments might be also exchanged among different parties to collaboratively train the model.
\end{itemize}

As analyzed above, we need a privacy-preserving algorithm which can prevent privacy leakage from both the word-count information and the sampled topics. For the former, we can introduce noise in each iteration to protect the word-counts. For the latter, we propose to utilize the inherent privacy of the CGS algorithm to achieve protection.

\subsection{Model Assumptions}
Here we give some assumptions on the adversary model and the neighboring datasets.
\subsubsection{Adversary Model}
We assume that the data curator is trustworthy and has no interest in inferring data privacy. But there exists a strong external adversary who can observe the sampled topic assignments $\textbf{K}$ and the word-count matrices $N_k^t$ and $N_m^k$ in each iteration of the training process. Based on the observed statistics, the adversary attempts to infer the private information of the training data.

\subsubsection{Neighboring Datasets}
 We construct the neighboring dataset $D'$ by replacing a single word $w_r=t\in D$ by $t'$ and call this process as \textit{word replacement}. And we expect to prevent the adversary from detecting the impact of this \textit{word replacement} on the CGS algorithm and thus stealing the sensitive information.

\subsection{Inherent Privacy of CGS Training Algorithm}
In the following, we will systemically study the inherent privacy of CGS-based LDA training algorithm.

\subsubsection{Basic Idea}\label{CGS-EM}
It has been shown that, Gibbs sampling process has some degree of inherent DP for free~\cite{foulds2016theory} since it works in the same way as an exponential mechanism for DP. As a variant, Collapsed Gibbs Sampling (CGS) naturally inherits this property.
Intuitively, CGS conducts random sampling iteratively to learn a topic-word distribution, which acts as a mechanism that probabilistically outputs a topic from the topic set. 
As described in Section~\ref{sec:DP and EM}, the exponential mechanism works by probabilistically sampling an answer from the answer set. This intuitive consistency motivates us to analyze CGS process in the view of the exponential mechanism.

Without loss of generality, we consider the sampling process of any word $w$ in the $i$-th iteration. Suppose its sampling distribution on $K$ topics is $\mathbf{P}=(p_1,p_2,...,p_K)^{\top}$, where $p_k$ represents the probability that topic $k$ is sampled. Define a function $u(w,k)=\log{p_k}$ with sensitivity $\Delta u_k \neq 0$, then $p_k$ could be written as
\begin{align}\label{intuition of inhenrent dp}
\nonumber	p_{k}&=\mathrm{exp}({\log p_{k}})=\mathrm{exp}({\frac{\Delta u_k}{\Delta u_k}\log p_{k}}) \\
&=\mathrm{exp}({\frac{\Delta u_k \cdot u(w,k)}{\Delta u_k}}).
\end{align}
According to Definition~\ref{defination: exponential mechanism}, Equation~(\ref{intuition of inhenrent dp}) can be viewed as the probability that an exponential mechanism $\mathcal{M}_E (w, u, \mathcal{K}):\mathcal{W}\rightarrow \mathcal{K}$ outputs the topic $k\in \mathcal{K}$ according to its utility $u(w,k)=\log{p_k}$. And $\mathcal{M}_E (w, u, \mathcal{K})$ provides $(\Delta u_k)$-DP. Notably, if taking the unnormalized probability $r_k$, e.g., $p_k\propto r_k$, to define the utility function as $u(w,k)=\log{r_k}$, then $\mathcal{M}_E (w, u, \mathcal{K})$ preserves $2(\Delta u_k)$-DP.

\begin{table}[!t]\centering\caption{Notations}
\begin{tabular}{ll}
	\hline
	\ $\alpha, \beta$      & hyper-parameters for LDA  \\
	\ $\mathcal{W}$   & word space\\
	\ $\mathcal{K}$   & topic space\\
	\ $\phi_k$      & topic-word distribution for topic $k$\\
	\ $\phi_k^t$    & probability that word $t$ is generated by topic $k$ \\
	\ $n_m^k$       & count of words with topic $k$ in document $d_m$ \\
	\ $n_k^t$       & count of word $t$ with topic $k$ in corpus $D$\\
    \ $n_t$        & total word count of $t\in D$ \\
	\ $p_k$        &probability that topic $k$ is sampled\\  
	\ $w_r$        &replaced word\\
	\ $\mathbf{w^{+}}$        &set of related words\\
	\ $\mathbf{w^{-}}$ & set of unrelated words\\
	\ $\epsilon^r_i$   & inherent privacy loss on $w_r$ in the $i$-th iteration\\
	\ $\epsilon^t_i$, $\epsilon^{t'}_i$ & inherent privacy loss on related words $t\,\text{or}\,t'$\\
	\hline
\end{tabular}
\label{tab:plain}
\end{table}
This observation enables us to calculate the inherent privacy of the CGS algorithm. To do so, we first investigate the inherent privacy loss in each iteration of the CGS training algorithm and then compose the privacy in total iterations by utilizing the composition theorem of DP.
\subsubsection{Inherent Privacy in Each Iteration}
In the $i$-th iteration, to bound the inherent privacy, it requires analyzing the impact of the \textit{word replacement} on each sampling. According to the relationship with the replaced word, we consider the following three cases for each word in $D$ respectively.
\begin{itemize}
\item \textit{Replaced word} is defined as the word $w_r=t \in D$ which is replaced by $t'$ in the neighboring dataset $D'$. For example, in Fig.~\ref{Relatedwords}, the word \textit{apple} marked in red is the \textit{Replaced word} in $D$.
\item \textit{Related words} refer to the words in $\textbf{w}^{+}$ which satisfies $\{w=t \,\text{or}\, w=t', \forall w\in \textbf{w}^{+}\}$. $\textbf{w}^{+}$ doesn't contain the \textit{Replaced word}. The words \textit{apple} and \textit{banana} marked in green shown in Fig.~\ref{Relatedwords} are \textit{related words} in $D$.
\item \textit{Unrelated words} are the words in $\textbf{w}^{-}$ which satisfies $\{w\neq t \,\text{and}\, w\neq t', \forall w\in \textbf{w}^{-}\}$. All the rest words in $D$ marked in black are unrelated words in the example shown in Fig.~\ref{Relatedwords}.
\end{itemize}

\begin{figure}[t]
\centering
\includegraphics[width=0.9\linewidth]{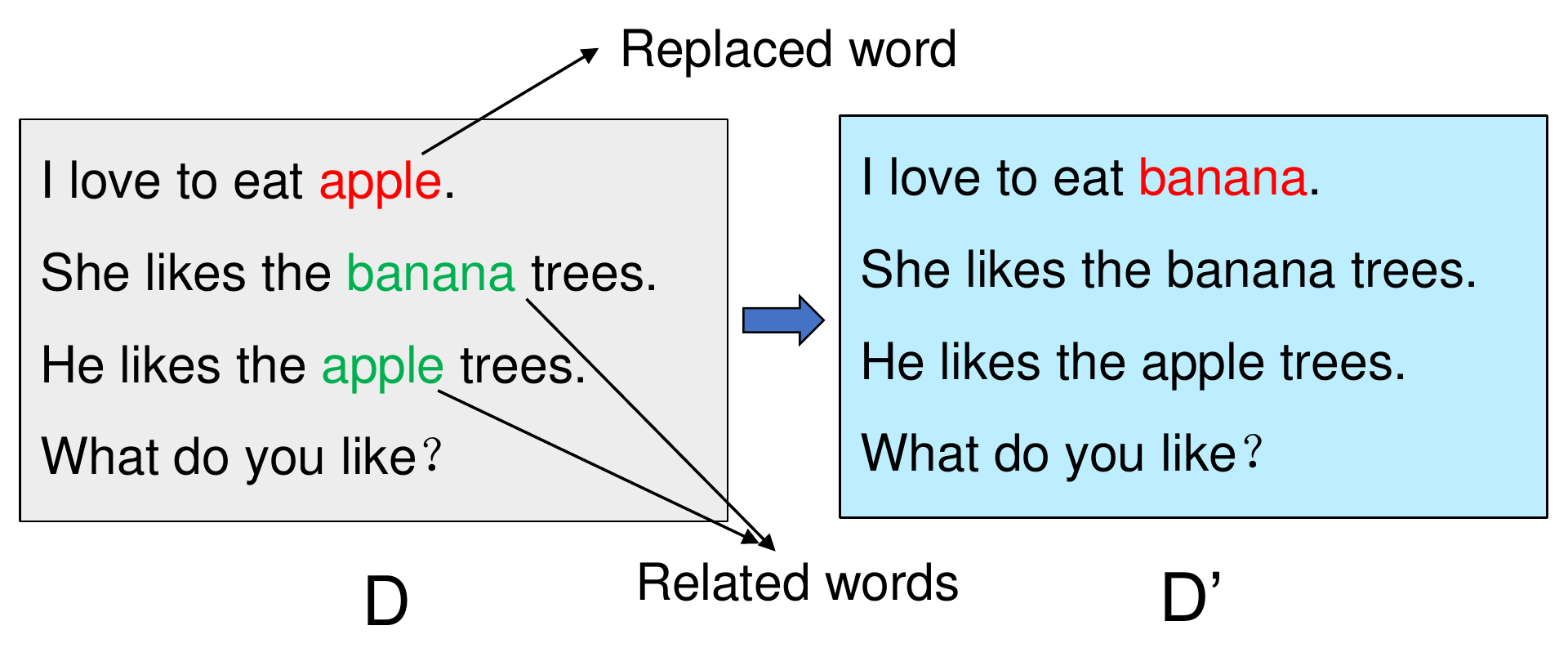}
\caption{An Example of Neighboring Datasets}
\label{Relatedwords}
\end{figure}

To analyze the privacy loss on the three types of words, the basic idea is to utilize the consistency between each topic sampling process and the exponential mechanism of DP as observed in Section~\ref{CGS-EM}. 

First, we analyze the privacy loss incurred by the topic sampling on the \textit{replaced word} $w_r$. 
In particular, in $D$, this sampling is performed on $w_r=t$ while in $D'$, this sampling is performed on $w_r=t'$. Here, we present a proposition to compute the incurred privacy loss in this sampling.
\begin{proposition}\label{proposition:privacy loss of part 1}
Suppose $D'$ is constructed from $D$ by replacing $w_r=t\in D$ by $t'$, then the privacy loss $\epsilon_i^r$ incurred by the topic sampling on the \textit{replaced word} $w_r$ in the $i$-th iteration can be bounded by
\begin{equation}\label{Equ:privacy loss on the replaced words}
\begin{aligned}
\epsilon_i^r \leq 2\max_k{\{|\log{\frac{n_{k}^{t'}+\beta}{n_k^t+\beta}}|\}},
\end{aligned}
\end{equation}
where $n_{k}^{t'}$ and $n_{k}^{t}$ represent the counts of topic $k$ assigned to $t'$ and $t$ in $D$, respectively.
\end{proposition}
\begin{proof}
Shown in Appendix~\ref{proof proposition part1}.
\end{proof}

Then, we consider the privacy loss incurred by the topic sampling on the \textit{related words} $\mathbf{w^{+}}$. For each word $w\in\mathbf{w^{+}}$, the \textit{word replacement} impacts the topic sampling process through the topic-word counts $n_k^t$ and $n_k^{t'}$, which would be used to compute the sampling distribution.

In particular, for some specific topic $k$, suppose the topic-word count for $t$ and $t'$ are $n_k^t$ and $n_k^{t'}$ respectively in $D$, the corresponding topic-word count would be $n_k^t-1$ and $n_k^{t'}+1$ in $D'$. Here, we also give a proposition to compute the inherent privacy loss.

\begin{proposition}\label{proposition:privacy loss of part 2}
The privacy loss in the CGS sampling process on each related word $w=t\in \mathbf{w}^{+}$ or $w=t'\in \mathbf{w}^{+}$ in the $i$-th iteration can be bounded by
\begin{equation*}
\begin{aligned}
\epsilon_i^{t} \leq 2\log{(1+\frac{1}{\beta})},\quad \epsilon_i^{t'} \leq 2\log{(1+\frac{1}{\beta})},
\end{aligned}
\end{equation*}
where $\epsilon_i^{t'}$ and $\epsilon_i^{t'}$ represent the privacy loss on $w=t'$ and $w=t$ in $\mathbf{w^{+}}$ respectively.
\end{proposition}
\begin{proof}
The proof logic is similar to that of Proposition~\ref{proposition:privacy loss of part 1}.
\end{proof}

Finally, we consider the \textit{unrelated words} $\mathbf{w^{-}}$. Actually, for each word $w=\hat{t}\in \mathbf{w^{-}}$, the \textit{word replacement} could not impact the sampling process on $w=\hat{t}$ .

So far, we have analyzed the impacts of \textit{word replacement} on the sampling process on all three types of words involved in the $i$-th iteration. Based on the impact analysis, the following theorem can bound the total privacy loss in the $i$-th iteration.
\begin{theorem}\label{theorem:composition in single iteration}
Given the \textit{replaced word} $w_r=t\in D$, the total privacy loss $\epsilon_i$ incurred by CGS algorithm $\mathcal{S}_i$ in the $i$-th iteration could be bounded by
\begin{equation}\label{Equ:privacy loss in a single iteration}
\begin{aligned}
\epsilon_i &\leq \epsilon_i^r+(n_t-1)\epsilon_i^{t}+n_{t'}\epsilon_i^{t'}
\\&\leq 2\{\log{(\frac{\max_{k,t'}{\{n_{k}^{t'}\}}}{\beta}+1)}
\\& +(\max_{t'}\{n_{t'}\}+n_{t}-1)\cdot\log{(1+\frac{1}{\beta})}\},
\end{aligned}
\end{equation}
where $n_t$ and $n_{t'}$  denote the total word-counts of $t$ and $t'$ appearing in $D$ respectively. $n_{k}^{t'}$ denotes the total count of topic $k$ assigned to $t'$ in $D$ when performing sampling on the \textit{replaced word} $w_r$.
\end{theorem}
\begin{proof}
Shown in Appendix~\ref{proofcomposition in single iteration}.
\end{proof}

Theorem~\ref{theorem:composition in single iteration} demonstrates that the inherent privacy loss in a single CGS iteration is the sum of the privacy losses on the \textit{replaced word} and all the \textit{related words}.
And as we can see, the inherent privacy loss can be adjusted by the hyper parameter $\beta$. The larger $\beta$, the less privacy loss, and vice versa. That is reasonable since $\beta$ represents the prior information for the topic-word count and a larger $\beta$ means that the sampling process is mainly determined by the prior information instead of the training data. Then, less data information can be inferred in the sampling process.

\subsubsection{Inherent Privacy in Total CGS Training}

So far, we have identified an upper bound for the privacy loss in a single iteration. Now, based on the composition theorem of DP, we can bound the total privacy loss in the whole training process with multiple iterations.
\begin{theorem}\label{theorem:composition in total CGS}
Suppose the privacy loss of each word $w$ in the $i$-th iteration is $\epsilon_i^w$, then the total privacy loss $\epsilon$ in the whole training process with $n$ iterations could be bounded by
\begin{equation}\label{Equ:privacy loss in total CGS}
\begin{aligned}
\epsilon \leq \max_{w\in D}{\{\sum_{i=1}^n{\epsilon_i^w}\}}.
\end{aligned}
\end{equation}

\end{theorem}
\begin{proof}
Shown in Appendix~\ref{proof composition in total CGS}.
\end{proof}

\subsubsection{Limitations of Inherent Privacy}\label{discussions on inherent privacy}
We have systemically analyzed the inherent privacy of the CGS training algorithm. Here, we present two limitations of the inherent privacy of CGS regarding the privacy protection.
\begin{itemize}
\item \textbf{Rapid Accumulation of Privacy Loss}. As shown in the second term of Equation~(\ref{Equ:privacy loss in a single iteration}), the privacy loss accumulates almost linearly with respect to the word count $n_t$ and $n_{t'}$. This is mainly because the \textit{word replacement} impacts not only the replaced word, but also all the related words in the sampling process. 
\item \textbf{Lack of Protection on Word-count Information}.
Obviously, the inherent privacy cannot protect the word-count information, e.g., $n_k^t$, which also reflects the raw data information. Since there is no noise introduced, the \textit{word replacement} would incur an unmaskable change on the word-counts, e.g., $n_k^t-1$ and $n_k^{t'}+1$, thus further causing the privacy leakage if we assume an adversary who can observe that.
\end{itemize}

\subsection{Hybrid Privacy-preserving LDA Training}

In Section~\ref{Limitations of the Existing method}, we summarized the limitations of the existing works in protecting the sampled topics, while in Section~\ref{discussions on inherent privacy}, we identified that the inherent privacy is lack of protection on the word-count information. To address these issues, we propose a hybrid privacy-preserving algorithm named \textsf{HDP-LDA} that combines the inherent privacy of CGS approach and external privacy based on noise injection. The basic idea is to introduce a proper noise in each iteration of CGS to protect the word-count information while mitigating the rapid privacy loss accumulation of inherent privacy. The detailed algorithm process is shown in Algorithm~\ref{alg: Hybrid privacy preserving Algorithm}. Two main operations in the algorithm are explained as follows.

\begin{itemize}
\item \textbf{Adding noise:} As analyzed, both limitations of inherent privacy are caused by the difference on the topic-word counts $n_k^t$ between $D$ and $D'$. Therefore, we introduce some noise to obfuscate the difference on $n_k^t$ in each iteration (Line 9 in Algorithm~\ref{alg: Hybrid privacy preserving Algorithm}). Theorem~\ref{chain effect mitigation theorem} proves that Algorithm~\ref{alg: Hybrid privacy preserving Algorithm} can mitigate the rapid privacy loss accumulation of inherent privacy.
\item \textbf{Clipping:} Quantifying the inherent privacy requires the knowledge of the upper bound of topic-word counts $\max_{k,t}{\{n_k^t\}}$, as shown in Equation~(\ref{Equ:privacy loss in a single iteration}). A natural bound for $n_k^t$ is the total count $n_t$ of word $t$, which, however, might be too loose if there are too many $K$ topics in total. Therefore, we resort to a clipping method to limit the inherent privacy in each iteration. Notably, this clipping is only performed on a copy of $n_k^t$ for the computation of sampling distribution (Line $12$), but does not impact the subsequent updating of $n_k^t$ in CGS (Lines 15 $\sim$ 16).
\end{itemize}

\begin{theorem}\label{chain effect mitigation theorem}
Algorithm~\ref{alg: Hybrid privacy preserving Algorithm} satisfies $(\epsilon_L+\epsilon_I)$-DP in each iteration, in which
\begin{equation}\label{Equ:chain effect mitigation}
\begin{aligned}
\epsilon_I=2\log{(\frac{C}{\beta}+1)}
\end{aligned}
\end{equation}
denotes the inherent privacy loss, $\epsilon_L$ denotes the privacy loss incurred by the Laplace noise, and $C$ denotes clipping bound for $n_{k}^{t}$.
\end{theorem}
\begin{proof}
Shown in Appendix~\ref{proof chain effect mitigation theorem}.
\end{proof}

As shown in Theorem~\ref{chain effect mitigation theorem}, the privacy loss in \textsf{HDP-LDA} consists of two parts: privacy loss $\epsilon_L$ incurred by the Laplace noise and the inherent privacy loss $\epsilon_I$ of CGS algorithm. Equation~(\ref{Equ:chain effect mitigation}) presents the inherent privacy loss. Comparing with Equation~(\ref{Equ:privacy loss in a single iteration}), we can see that the rapid accumulation of inherent privacy loss has been mitigated since the second term of Equation~(\ref{Equ:privacy loss in a single iteration}) has been removed.

\begin{algorithm}[t]
\SetAlgoLined
\caption{HDP-LDA}
\label{alg: Hybrid privacy preserving Algorithm}
\LinesNumbered
\SetKwComment{Comment}{start}{end}
\KwIn{Document corpus $D$, Prior parameters $\alpha$, $\beta$, Topic number $K$, Clipping bound $C$}
\KwOut{Trained topic-word distribution $\Phi$, Privacy loss $\epsilon=T\cdot(\epsilon_L+\epsilon_{I})$}
\tcp{\textbf{Initialization}}
\For {$d_m\in D$}{
	\For {$w=t\in d$}{
		Sample topic: $k\sim Mult(\frac{1}{K}\cdot\textbf{I}_K)$\;
		Initialize word counts $n_k^t$ and $n_m^k$\;
	}
}
\tcp{\textbf{Collapsed Gibbs Sampling}}
Set $iter=0$\;
\While {$iter<T$} {
	Add noise to each $n_k^t$ independently: $n_k^t\leftarrow n_k^t+\eta,\quad \eta\sim Lap(2/\epsilon_L)$\;
	\For {$d\in D$}{
		\For {$w=t\in d$}{
			Clip: $(n_{k}^{t})^{temp}\leftarrow \min\{n_{k}^{t}, C\}$\;
			Compute sampling distribution $\mathbf{p}$:
			$p_k \propto \frac{(n_{k}^{t})^{temp}+\beta}{\sum_{t=1}^{V}(n_k^t+\beta)}\cdot\frac{n_{m }^{k}+\alpha}{\sum_{k=1}^{K}(n_{m}^{k}+\alpha)}$\;
			Compute inherent privacy loss: $\epsilon_I \leftarrow 2\log{(\frac{C}{\beta}+1)}$\;
			Sample topic and update word count $n_k^t$\;
		}
	}
	$iter \leftarrow iter+1$\;
}
  Compute the trained model $\Phi$\;
\end{algorithm}

\section{\textsf{LP-LDA}: LDA Model Training with LDP}\label{sec:LP-LDA}

We have developed a comprehensive privacy protection approach for protecting the CGS training process on a centralized curated dataset, where the central server is assumed trustworthy. Nevertheless, in many distributed applications, the data curator may not be reliable and the individual data owners may be reluctant to share their sensitive data. In this case, we propose an LDP solution of \textsf{LP-LDA} for LDA that can train on crowdsourced data with LDP. \textsf{LP-LDA} is constituted by two parts: local perturbation at the user side and training on reconstructed dataset at the server side.

\subsection{Local Perturbation}\label{sec:localper}

The local perturbation at the user side includes the following steps:
\begin{itemize}
\item \textit{Step 1}. Each document $m$ is encoded as a binary vector $\mathbf{V}_m$, in which each bit $\mathbf{V}_m [j]$ represents the presence of the $j$-th word in the word bag of the corpus.
\item \textit{Step 2}. Each bit $\mathbf{V}_m [j]$ of the binary vector $\mathbf{V}_m$ is then randomly flipped according to the following randomized response rule:
\begin{align*}
\hat{\mathbf{V}}_m [j]=
\begin{cases}
\mathbf{V}_m [j], ~~& \text{with probability of}~ 1-f\\
1, ~~& \text{with probability of}~ f/2\\
0, ~~& \text{with probability of}~ f/2\\
\end{cases}
\end{align*}
where $f \in [0,1]$ is a parameter that specifies the randomness of flipping and adjusts the local privacy level.
\item \textit{Step 3}. Then the noisy binary vector $\hat{\mathbf{V}}_m [j]$ is sent to the central server by each user. Obviously, $\hat{\mathbf{V}}_m [j]$ is locally sanitized without concerning user's privacy.
\end{itemize}

\subsection{Training on the Reconstructed Dataset}\label{sec:reconstruct}

After receiving the flipped binary vectors from a large number of data contributors, the central server can aggregate the vectors, reconstruct the dataset, and then perform training on the reconstructed dataset. The rationale behind this is that the training result of topic-word distribution is insensitive to the document partitions and only depends on the total word counts in the corpus.
\begin{itemize}
\item \textit{Step 1}. For each bit in the noisy binary vectors, the server counts the number of $1'$s as $n_t=\sum_{m=1}^M \hat{\mathbf{V}}_m [t]$.
\item \textit{Step 2}. The server then estimates the true count $N_t$ of each bit in the original binary vectors $\mathbf{V}_m$  as $\hat{N}_t=(2n_t-f M)/2(1-f)$.
\item \textit{Step 3}. For each bit, the server first computes the difference $\delta_t=\hat{N}_t-n_t$.
\item \textit{Step 4}. For each bit $t$, if $\delta_t>0$, the server randomly samples $\delta_t$ binary vectors with the $t$-th bit as $0$ and sets the $t$-th bit as $1$; if $\delta_t<0$, then the server randomly samples $|\delta_t|$ binary vectors with the $t$-th bit as $1$ and sets the $t$-th bit as $0$; otherwise, keeps the noisy bit vectors as received.
\item \textit{Step 5}. Based on the noisy bit vectors, the server reconstructs a dataset and performs the CGS process on it.
\end{itemize}
Algorithm \ref{alg:LP-LDA} presents the detailed procedures of LP-LDA on both the user side and the server side.

\subsection{Privacy Analysis of \textsf{LP-LDA}}

\begin{theorem}
\textsf{LP-LDA} satisfies $\log\frac{1-f/2}{f/2}$-LDP for each word, and $V\cdot \log\frac{1-f/2}{f/2}$-LDP for each document. 
\end{theorem}
\begin{proof}
Suppose a word $t$ appears in a noisy bit vector, then the probability of it being kept from the original bit vector is $\mathrm{Pr}(\hat{\mathbf{V}}_m [t]=1|\mathbf{V}_m [t]=1)=1-f/2$ and the probability of it being flipped from the original bit vector is $\mathrm{Pr}(\hat{\mathbf{V}}_m [t]=1|\mathbf{V}_m [t]=0)=f/2$. Then, according to the definition, it guarantees the local differential privacy of
\begin{equation*}
\setlength{\abovedisplayskip}{3pt}
\setlength{\belowdisplayskip}{3pt}
\begin{aligned}
\epsilon=\left|\log\frac{\mathrm{Pr}(\hat{\mathbf{V}}_m [t]=1|\mathbf{V}_m [t]=1)}{\mathrm{Pr}(\hat{\mathbf{V}}_m [t]=1|\mathbf{V}_m [t]=0)}\right|=\log\frac{1-f/2}{f/2}.
\end{aligned}
\end{equation*}
The analysis holds for any bit $t$ that $\hat{\mathbf{V}}_m [t]=0$.

Each bit of the document $\mathbf{V}_m$ is perturbed independently. Then according to the sequential composition theorem of DP, \textsf{LP-LDA} guarantees $V\cdot\epsilon$-local DP for the entire document.
\end{proof}
Since the reconstruction and training process are essentially post-processes on the noisy bit vectors, the LDP guarantee remains unchanged for all the documents.

\subsection{Utility Analysis of \textsf{LP-LDA}}
\setlength{\abovedisplayskip}{6pt}
\begin{theorem}
Let $N_t$ and $n_t$ denote the counts of word $t$ in the original and perturbed datasets, respectively, then
\begin{equation}\label{moment estimator}
\setlength{\abovedisplayskip}{3pt}
\setlength{\belowdisplayskip}{3pt}
\hat{N_{t}}=\frac{2n_t-fM}{2(1-f)}
\end{equation}
is an unbiased estimator of $N_t$ with the variance of
\begin{equation}\label{variance of moment estimator}
\setlength{\abovedisplayskip}{3pt}
\setlength{\belowdisplayskip}{3pt}
\mathrm{Var}(\hat{N_{t}})=\frac{(2-f)fM}{4(1-f)^2}+\frac{(M-N_t)N_t}{M}.
\end{equation}
\end{theorem}
\begin{proof}
Let $n_1$ denote the count of word $t$ retained from the real datasets and $n_2$ denote the noisy part, then $n_1$ and $n_2$ follow two Binomial distributions, i.e.,
$n_{1}\sim B(N_{t}, 1-f/2)$, $n_{2}\sim B(M-N_{t}, f/2)$. Let $ X=n_{1}+n_{2}$, then its first theoretical moment $\mathbb{E}(X)=N_{t}(1-f/2)+(M-N_{t})\cdot(f/2)$ and its first sample moment $\bar{X}= n_{t}$. Therefore,
\begin{equation*}
\hat{N_{t}}=\frac{2n_t-fM}{2(1-f)}
\end{equation*}
is the moment estimator as well as an unbiased estimator. Its variance is then
\begin{equation*}
\begin{aligned}
\mathrm{Var}(\hat{N_{t}})&=\frac{\mathrm{Var}(n_t)}{(1-f)^2}=\frac{ \mathrm{Var}(n_{1}+n_{2})}{(1-f)^2}\\&=\frac{ \mathrm{Var}(n_{1})+\mathrm{Var}(n_{2})+2\mathrm{Cov}(n_1,n_2)}{(1-f)^2}\\&=\frac{(2-f)fM}{4(1-f)^2}+\frac{(M-N_t)N_t}{M}.
\end{aligned}
\end{equation*}
\end{proof}

\section{\textsf{OLP-LDA}: Online LDA Model Training with LDP}\label{sec:OLP-LDA}

\subsection{Motivations}
In \textsf{LP-LDA}, we consider a static scenario that one-time LDA is trained on the locally sanitized dataset. Here, we consider the following two practical issues that may encounter in LDA model training.
\begin{itemize}
\item \textbf{Online Training}. Many practical applications require that ML models can be continuously trained on streaming datasets, which are contributed by users in mini-batches and accumulated over time. However, it is usually infeasible to store all streaming batches and perform batch model training such as LP-LDA.
\item \textbf{Prior Knowledge}. In many real-world scenarios, the training server may not begin model training from zero but possess some prior datasets for building an initial model. These prior datasets may be from the publicly available datasets or the purchased high-quality crowdsourced datasets, which are often non-private.
\end{itemize}

Considering the above issues, we further propose \textsf{OLP-LDA}, a privacy-preserving online LDA training algorithm. \textsf{OLP-LDA} aims to not only realize efficient online training on mini-batches with LDP, but also greatly improve the model utility by extracting knowledge from the prior dataset.

\textsf{OLP-LDA} consists of two components: the baseline online LDA training framework \textsf{O-LDA}, and the Bayesian denoising technique for  the continuous reconstruction of noisy batches.

\subsection{\textsf{O-LDA}: Framework of Online LDA Model Training}
\subsubsection{Basic Idea of Online LDA Training}
Given the prior information $P(\Theta|D_0)$ of LDA model parameter $\Theta$ from the prior dataset $D_0$, the Online LDA training aims to update the LDA model $P(\Theta|D)$ with the evolving mini-batch sequence $D_{1:L}=\{D_{1},...,D_{l},...,D_{L}\}$ where $D_{l}$ represents the $l-$th mini-batch. Considering the correlations between mini-batches, the online training process could be regarded as a Bayesian learning process, in which $P(\Theta|D)$ is updated based on a recurrence relationship:
\begin{equation*}
P(\Theta|D_{0:l})\propto P(\theta|D_{0:l-1})P(D_{l}|\Theta),
\end{equation*}
where the posterior $P(\Theta|D_{0:l-1})$ learned from $D_{0:l-1}$ would be used as the prior when learning from $D_{l}$.

\subsubsection{\textsf{O-LDA} Generative Model}	
To capture the correlations between consecutive mini-batches $D_{l-1}$ and $D_{l}$, we introduce a correlation factor $\lambda$. And the prior parameters $\beta^l$ for $D_{l}$ is represented as the combination of $\beta^{l-1}$ for $D_{l-1}$ and the topic-word matrix $N_{k,t}^{l-1}$ learned from $D_{l-1}$. That is
\begin{equation}\label{equation: prior update}
\beta^{l}=\beta^{l-1}+\lambda N_{k,t}^{l-1}.
\end{equation}
In Equation~(\ref{equation: prior update}), a larger $\lambda$ means a stronger dependency between $D_{l-1}$ and $D_{l}$. When $\lambda=0$, the training result of $D_{l-1}$ will not influence the training process of $D_{l}$. Apparently, $\beta^{l}$ could also be written as
\begin{equation}
\beta^{l}=\beta \mathbf{1}_{K,V}+\sum_{i=1}^{l-1}\lambda N_{k,t}^{i},
\end{equation}
where $\beta$ denotes a hyperparameter fixed for the whole corpus.

As a result, we redefine the corpus generative process of the \textsf{O-LDA} model as follows:

\begin{enumerate}
\item Generate the first mini-batch $D_{1}$ according to the standard LDA with given hyperparameters $\alpha$ and $\beta$, and draw a correlation factor $\lambda\sim U(0,1)$ for the corpus.
\item For each mini-batch $D_{l}, l\geq2$, generate the topic-word distribution $\Phi^l\sim Dir(\beta^l)$ for $D_{l}$.
\item For each document $m$ in mini-batch $D_{l}$, generate all the words according to the standard LDA with hyperparameter $\alpha$ and topic-word distribution $\Phi^l$.
\end{enumerate}

Then, the sampling distribution for word $w=t$ in mini-batch $D_{l}$ shown in Equation~(\ref{Equ:sampling equation}) should be replaced as
\begin{equation}
\begin{aligned}
p(z_w=k)\propto \frac{n_{k,t}^{l}+\beta^{l}_{k,t}}{\sum_{t=1}^{V}(n_{k,t}^{l}+\beta^{l}_{k,t})}\cdot\frac{n_{m}^{k}+\alpha}{\sum_{k=1}^{K}(n_{m}^{k}+\alpha)},
\end{aligned}
\end{equation}
where $\beta^{l}_{k,t}$ denotes the element in the $k$-th row and $t$-th column of $\beta^{l}$. Similar to the standard LDA, the training result $\Phi$ of topic word distribution is updated by
\quad
\begin{equation}
\mathbb{E}[\phi_k^t|D_{1:l}]=\frac{n_{k,t}^{l}+\beta^{l}_{k,t}}{\sum_{t=1}^{V}(n_{k,t}^{l}+\beta^{l}_{k,t})}.
\end{equation}
\quad

The parameter $\lambda$ could be optimized with each mini-batch to reach a better performance using the MCMC method~\cite{amoualian2016streaming}.

\subsection{Bayesian Denoising Scheme}

\subsubsection{Basic Idea of Bayesian Denoising}	
Recall that the locally sanitized dataset can be reconstructed based on the moment estimation of the word counts in LP-LDA (Section~\ref{sec:reconstruct}). Bayesian denoising scheme aims to further improve this estimation, i.e., reducing the variance in Equation~(\ref{variance of moment estimator}), with the knowledge in the prior dataset.
The denoising process can be also viewed as a Bayesian estimation problem, where the word count $N_t$ is the underlying parameter for a prior distribution $\pi(N_t)$ and the observation $n_t$. The objective is to find an estimation function $B(n_{t})$ of $n_t$ to minimize the following Bayes risk:
\begin{equation*}
R(\pi(N_t))=\min_{B(n_{t})}\mathbb{E}_{N_{t}}[\mathbb{E}_{n_{t}|N_{t}}[\Vert{B(n_{t})-N_{t}}\Vert^2|N_{t}]].
\end{equation*}
According to~\cite{lehmann2006theory}, the optimal estimator can be calculated as the posterior expectation
\begin{equation*}
\mathbb{E}[N_t|n_t]=\mathop{\arg\min}_{B(n_{t})}R(\pi(N_t))
\end{equation*}
which depends on the prior distribution $\pi(N_t)$ and the likelihood function $P(n_t|N_t)$.

\subsubsection{Prior Distribution}
The prior distribution $\pi(N_t)$ can be assumed to follow a Gaussian distribution $N(\mu_t, \sigma^2)$ for simplicity, where $\mu_t$ denotes the word count information extracted from the prior dataset and $\sigma^2$ represents the level of belief to the prior information. Gaussian distribution is a common assumption and $\sigma^2$ could be adjusted according to the practical requirement.	

\subsubsection{Likelihood Function}
The likelihood function $P(n_t|N_t)$ can be computed as follows. As did in \textsf{LP-LDA}, each bit of the binary word vector transformed from the document has a probability of $f/2$ to be flipped at the user side, so the likelihood function of the observed noisy word count could be written as
\begin{equation}\label{ordinary likelihood}
\begin{aligned}
P(X=n_t|N_t)&=\sum_{i=0}^{min\{N_t,n_t\}}p(X_1=i,X_2=n_{t}-i)\\
&=\sum_{i=0}^{min\{N_t,n_t\}}\binom{N_t}{i}(1-\frac{f}{2})^{i}(\frac{f}{2})^{N_{t}-i}\times\\&\binom{M-N_t}{n_{t}-i}(1-\frac{f}{2})^{M-N_{t}-n_{t}+i}(\frac{f}{2})^{n_{t}-i},
\end{aligned}
\end{equation}
where $X_1$ denotes the count of word $t$ retained from the real datasets, $X_2$ denotes the noisy part and $X=X_1+X_2$, and $M$ denotes the document number in this mini-batch. Note that $X_1$ and $X_2$ follow two Binomial distributions
\begin{equation}\label{equation: both binomial}
X_1\sim B(N_t, 1-f/2)\quad X_2\sim B(M-N_t, f/2).
\end{equation}

It is still intractable to present the posterior $\mathbb{E}[N_t|n_t]$ in a closed form due to the complicated likelihood function in Equation~(\ref{ordinary likelihood}). To tackle this problem, we consider deriving an approximate results based on the following Gaussian conjugate property.

\begin{lemma}\label{lemma:Gauss}(Gaussian Conjugate Property~\cite{bishop2006pattern})
Given a marginal Gaussian distribution $p(X)$ and a conditional Gaussian distribution $p(Y|X)$, the conditional distribution $p(X|Y)$ satisfies that
\begin{equation*}
p(X|Y)=N(X|w\mu_{0}+(1-w)\frac{Y-b}{a},\frac{\sigma_0^2\sigma_1^2}{\sigma_1^2+a^2\sigma_0^2}),
\end{equation*}	
if it holds that,
\begin{equation*}
p(X)=N(X|\mu_{0},\sigma_0^2),\quad p(Y|X)=N(Y|aX+b,\sigma_1^2),
\end{equation*}	
where $w=\frac{\sigma_1^2}{\sigma_1^2+a^2\sigma_0^2}$.
\end{lemma}

Lemma \ref{lemma:Gauss} implies that the posterior of a parameter $\theta$ will be also a Gaussian distribution when the prior distribution of $\theta$ and the likelihood function of the observed data conditioned on $\theta$ are based on Gaussian distribution. Based on Lemma \ref{lemma:Gauss}, we consider approximating the binomial distributions shown in Equation~(\ref{equation: both binomial}) using Gaussian distributions. The following lemma provides theoretical support for such approximation.

\begin{lemma}\label{lemma:CLT}(De Moivre-Laplace Central Limit Theorem)
Suppose that in a series of $n$ independent Bernoulli trials, event A has a probability $p$ of occurrence ($0<p<1$) in each trial. Denote $S_n$ as the occurrence time of A, and let
\begin{equation*}
Y_n^{*}=\frac{S_n-np}{\sqrt{npq}},
\end{equation*}
where $q=1-p$, Then for each real number $y$, it holds that
\begin{equation*}
\lim_{n\to \infty}P(Y_n^{*}\leq y)=\frac{1}{\sqrt{2\pi}}\int_{-\infty}^{y}e^{-t^2/2}dt.
\end{equation*}
\end{lemma}

Lemma \ref{lemma:CLT} states that a binomially distributed random variable approximates the Gaussian random variable with mean $np$ and standard deviation $\sqrt{npq}$ as $n$ grows larger. It has been validated that when $np>5$ or $n(1-p)>5$, this approximation is reasonable and effective.
Based on Lemma \ref{lemma:CLT}, both binomial distributions shown in Equation (\ref{equation: both binomial}) could be approximated as:
\begin{equation}\label{gaussian approximation}
\begin{aligned}
X_1&\sim N(N_t(1-\frac{f}{2}), N_t\cdot \frac{f}{2}(1-\frac{f}{2}))&\\X_2&\sim N((M-N_t)\frac{f}{2},(M-N_t)\frac{f}{2}(1-\frac{f}{2})).
\end{aligned}
\end{equation}

Furthermore, according to the additive property of Gaussian distributions, the word count variable $X$ should follow a Gaussian distribution
\begin{equation}\label{replaced likelihood}
\begin{aligned}
X\sim N((M-N_{t})\frac{f}{2}+N_{t}(1-\frac{f}{2}), \sigma_{p}^2),
\end{aligned}
\end{equation}
where $\sigma_{p}^2=M\frac{f}{2}(1-\frac{f}{2})+2cov(X_1,X_2)$ and then the likelihood function shown in Equation~(\ref{ordinary likelihood}) could be approximated as:
\begin{equation}\label{replaced likelihood}
\begin{aligned}
P(X=n_t|N_t)=\frac{1}{\sqrt{2\pi\sigma_p^2}}\mathrm{exp}(-\frac{[2X-Mf-2N_t(1-f)]^2}{4\sigma_p^2}).
\end{aligned}
\end{equation}
\begin{theorem}\label{theorem:bayesian denoising}
Suppose the prior distribution of count $N_t$ of word $t$ is set as $N(\mu_t,\sigma^2)$, and the likelihood function of the noisy count $n_t$ is approximated by Gaussian distribution, then the posterior distribution of $N_t$ should be as follows:
\begin{equation*}
\begin{aligned}
N(\omega\cdot \mu_t+(1-\omega)\frac{2n_t-fM}{2(1-f)},\frac{\sigma^2\sigma_p^2}{\sigma_p^2+(1-f)^2\sigma^2}),
\end{aligned}
\end{equation*}
and then the Bayesian estimator
\begin{equation}\label{equation:bayesian denoising}
\begin{aligned}
B(n_t)=\omega\cdot \mu_t+(1-\omega)\frac{2n_t-fM}{2(1-f)},
\end{aligned}
\end{equation}
where $\omega=\frac{Mf(2-f)}{Mf(2-f)+4\sigma^2(1-f)^2}$.
\end{theorem}
\begin{proof}
Based on Equation~(\ref{replaced likelihood}) and Lemma 2, the result could be derived directly.
\end{proof}
Recall the moment estimator of the word count $N_t$ shown in Equation~(\ref{moment estimator}), we can observe that the Bayesian estimator  of $N_t$ is a linear combination of the moment estimator and the prior $\mu_t$. And the weight $\omega$ of $\mu_t$ is
controlled by the prior parameter $\sigma^2$ which is adjustable according to practical requirement. Intuitively, as $\sigma^2$ grows larger, the belief to the prior information becomes weaker, and the weight of $\mu_t$ becomes smaller.

\begin{theorem}\label{theorem:variance of bayesian denoising}
The Bayesian estimator $B(n_t)$ of $N_t$ in Equation~(\ref{equation:bayesian denoising}) has a variance of
\begin{equation}\label{variance of bayesian estimator}
\begin{aligned}
\mathrm{Var}(B(n_t))=(1-\omega)^2(\frac{(2-f)fM}{4(1-f)^2}+\frac{(M-N_t)N_t}{M}),
\end{aligned}
\end{equation}
where $\omega=\frac{Mf(2-f)}{Mf(2-f)+4\sigma^2(1-f)^2}$.
\end{theorem}
\begin{proof}
Since $\mu_t$ is a constant, then
\begin{equation*}
\mathrm{Var}(\omega\cdot \mu_t+(1-\omega)\hat{N}_t)=(1-\omega)^2\cdot D(\hat{N}_t),
\end{equation*}
where $\hat{N}_t$ denotes the estimator of $N_t$ in Equation~(\ref{moment estimator}).
\end{proof}

Theorem~\ref{theorem:variance of bayesian denoising} reveals that if we replace the moment estimator in Equation~(\ref{moment estimator}) by the Bayesian estimator in Equation~(\ref{equation:bayesian denoising}) to implement the dataset reconstruction process, the variance term could be reduced to $(1-\omega)^2$ of that of the moment estimator.

\subsection{\textsf{OLP-LDA} Algorithm}
We present the online LDA algorithm on locally private mini-batches by embedding the Bayesian denoising process into the training procedures of \textsf{O-LDA}.

Let $D_0$ be the prior dataset, $D_l$ be the $l$-th locally sanitized mini-batch arriving at the central server, $N_t^l$ be the real count of word $t$ in $D_l$, $\mu_t^l$ be the prior parameter for $N_t^l$, shown as the mean of the prior distribution, and $\eta_t^{l}$ be the Bayesian estimation $B(n_t^{l})$ of $N_t^l$ in Equation~(\ref{equation:bayesian denoising}). We give the updating rule of $\mu_t^l$ as follows:

\begin{equation}\label{equation: prior updating}
\mu_t^l=\frac{\mu_t^{l-1}+\eta_t^{l-1}}{2}\times \frac{|D_{l}|}{|D_{l-1}|},
\end{equation}
where $|D_{l}|$ refers to the documents number in mini-batch $D_{l}$. Especially, it holds that $\mu_t^{0}=\eta_t^{0}$ since $D_0$ is possessed by the central server as the clean dataset and no need to be reconstructed.
Algorithm~\ref{alg:algorithm3} shows the details of our proposed \textsf{OLP-LDA}.
\subsection{Analysis of \textsf{OLP-LDA}}
We discuss the impacts of the mini-batch size and the prior data size on the model utility in terms of Bayesian denoising.
\subsubsection{Impact of Mini-batch Size}
Given a mini-batch $D_l$, the Bayesian estimation $B(n_t^l)=\frac{B(n_t^l)}{|D_l|}|D_l|$ of $N_t^l$ is used to perform the denoising. Obviously, more accurate frequency estimation $\frac{B(n_t^l)}{|D_l|}$ would lead to more accurate denoising. According to Equation~(\ref{variance of bayesian estimator}), the variance of $\frac{B(n_t^l)}{|D_l|}$ should be:
\begin{equation}\label{Equa:impact of mini-batchsize}
\mathrm{Var}(\frac{B(n_t^l)}{|D_l|})=(1-\omega)^2(\frac{(2-f)f}{4(1-f)^2|D_l|}+\frac{(|D_l|-N_{t}^l)N_{t}^l}{|D_l|^3}),
\end{equation}
where
\begin{equation*}
(1-\omega)^2=(\frac{4\sigma^2(1-f)^2}{f(2-f)|D_l|+4\sigma^2(1-f)^2})^2.
\end{equation*}
As seen, the larger $|D_l|$, the smaller variance of $\frac{B(n_t^l)}{|D_l|}$, thus the more accurate denoising.
\subsubsection{Impact of Size of Prior Dataset}	
Given a prior dataset $D_0$, if we take a random variable $X^i_t\in\{0,1\}$ to represent whether a given word $t$ appears in the $i$-th document in $D_0$, then according to the law of large numbers, the mean of $X^i_t$ in the prior dataset should tend to the real frequency statistic $p_t$ in the whole corpus as the size $D_0$ keeps increasing, that is
\begin{equation}\label{Equa:impact of priordatasize}
\lim_{|D_0|\to\infty}\mathrm{Pr}\{|\frac{1}{|D_0|}\sum_{i=0}^{|D_0|}X^i_t-p_t|<\epsilon\}=1, \forall \epsilon>0,
\end{equation}
where $\frac{1}{|D_0|}\sum_{i=0}^{|D_0|}X^i_t=\frac{N_t^0}{|D_0|}$. That means that the larger prior dataset can extract more accurate word-count information. Then, according to Equation~(\ref{equation: prior updating}) and $\mu_t^{0}=\eta_t^{0}$, the more accurate prior information would be generated for the subsequent noisy mini-batches.

Therefore, we can conclude that both the larger size of the prior dataset and  mini-batches could further improve the model utility.

\begin{algorithm}[t]
\SetAlgoLined
\caption{\textsf{OLP-LDA}}
\label{alg:algorithm3}
\LinesNumbered
\SetKwComment{Comment}{start}{end}
\KwIn{Flipping propablity $f$, Hyperparameters $\alpha$, $\beta$, Topic number $K$, Correlation factor $\lambda$, Reconstruction factor $\omega$, Prior dataset $D_0$, Local dataset $D$}
\KwOut{Training results $\Phi^{0:\infty}$}
\tcp{\textbf{On the user side}}
\ForEach{document $d \in D$}{
	$\hat{d}=RR(d)$ \tcp*{Randomized response}
	Upload $\hat{d}$ to the server \;
}
\tcp{\textbf{On the server side}}
\setcounter{AlgoLine}{0}
Train initial model: $\Phi^{0},N_{k,t}^0=LDA(D_0, \alpha, \beta, K)$ \tcp*{Standard LDA training algorithm}
\ForEach{word $t$}{
	Initialize $\mu^0_t=\eta^{0}_t=\sum_{k}N_{k,t}^0$ \;
}
\For {l=$1:\infty$} {
	\For {$\hat{d}$ uploaded in the $l$-th time window}{
		Aggregate $\hat{d}$: $D_l=\{\hat{d_1},\hat{d_2},...,\hat{d}_{N_l}\}$ \;
		Generate prior for $D_l$ acc. to Equation~(\ref{equation: prior updating}) \;
		Reconstruct $D_l$ by Bayesian denoising scheme acc. to  Equation~(\ref{equation:bayesian denoising})
		\;
		Update $\beta=\beta+\lambda N_{k,t}^{l-1}$ \;
		Train model: $\Phi^{l},N_{k,t}^l=LDA(\alpha, \beta, K)$ \;
		\ForEach{word $t$}{
			Compute $\mu^l_t=\eta^{l-1}_t$\;
			Compute $\eta^{l}_t=B(n_t)$ acc. to Equation~(\ref{equation:bayesian denoising})\;
		}
	}
}		
\end{algorithm}

\section{Evaluation}\label{sec:experiment}

In this section, we conducted extensive simulation experiments on four real-world datasets to evaluate the effectiveness of our proposed algorithms.

\subsection{ Experiment Setup}
\textbf{Datasets:}
Four real-world datasets were used in our experiments:
\begin{itemize}
\item \textsf{KOS}\footnote{\url{http://archive.ics.uci.edu/ml/datasets/KOS}} contains $3, 430$ blog entries from dailykos website.
\item \textsf{NIPS}\footnote{\url{http://archive.ics.uci.edu/ml/datasets/NIPS}} contains $1, 500$ research papers from NIPS conference.
\item \textsf{Enron}\footnote{\url{http://archive.ics.uci.edu/ml/datasets/Enron}} contains email messages from about $150$ users. The first $10, 000$ documents are extracted for experiments.
\item \textsf{FVENA}
\footnote{\url{http://archive.ics.uci.edu/ml/datasets/Victorian+Era+Authorship+Attribution}} contains books written by $50$ authors in the 19th century. $1, 800$ documents from the first and eighth authors ($900$ for each) are extracted for experiments.

\end{itemize}
We extracted part of these datasets as our training datasets and the rest as the testsets. For simplicity, we conducted a pre-processing on these datasets, in which, all stop words were removed, and $1,000$ most frequent words in each dataset were chosen as the corresponding vocabulary list. Details about the datasets after pre-processing can be found in Table~\ref{tab:plain}.

\begin{table}
\centering\caption{Details of the Datasets}
\begin{tabular}{cccc}
	\hline
	Dataset  & $\#.$ words  & $\#.$ training docs & $\#.$ test docs\\
	\hline
	\ KOS       & 209169     &3000 &430 \\
	\ NIPS         &410753       &1350 &150\\
	\ Enron         &356363       &8000 &2000\\
	\ FVENA         &386061         &1350 &450 \\
	\hline
\end{tabular}
\label{tab:plain}
\end{table}

\textbf{Simulation Methodology:}

For \textsf{HDP-LDA}, we designed a simple \textit{Topic-based Attack} algorithm to validate its performance on protecting the sampled topics. Specifically, for the $n$-th word $w_{mn}=t$ in document $d_m$, we utilized all the sampled topics $\{k_1,...,k_i,...,k_T\}$ on $w_{mn}$ in $T$ iterations and the sanitized word-counts $\mathbf{s}=\{\hat{(n_k^t)}_i, i\leq T\}$ to infer $w_{mn}$. The detailed algorithm is shown in Algorithm~\ref{alg: Topic Attack Algorithm}. In our experiment, for simplicity, we selected the last word in $D$ as the attacked word and took the inferred probability $P[w_{mn}=t|k_1,...,k_i,...,k_T]$ computed by Equation~(\ref {Equ:Attack equation}) as the attack accuracy.

\begin{equation}\label{Equ:Attack equation}
\begin{aligned}
&Pr(w_{mn}=t|k_1,...,k_i,...,k_T)
\\&=\frac{Pr(w_{mn}=t, k_1,...,k_i,...,k_T)}{Pr(k_1,...,k_i,...,k_T)}
\\&=\frac{\prod_i Pr(w_{mn}=t, k_i)}{\prod_i Pr(k_i)}=\prod_i \frac{Pr(w_{mn}=t, k_i)}{Pr(k_i)}
\\&\approx\prod_i \frac{(\hat{n_{k_i}^t})_i/\sum_{k,t}(\hat{{n_{k}^t}})_i}{\sum_t{(\hat{n_{k_i}^t})_i}/\sum_{k,t}(\hat{{n_{k}^t}})_i}=\prod_i \frac{(\hat{n_{k_i}^t})_i}{\sum_{t}{\hat{(n_{k_i}^t)}_i}}.
\end{aligned}
\end{equation}

\begin{algorithm}[t]
\SetAlgoLined
\caption{Topic-based Attack Algorithm}
\label{alg: Topic Attack Algorithm}
\LinesNumbered
\SetKwComment{Comment}{start}{end}
\KwIn{The attacked word $w_{mn}$}
\KwOut{The inferred result $\bar{t}$}
Set $i=0$\;
\While {$i<T$} {
	Add noise: $\hat{(n_k^t)}_i\leftarrow n_k^t+\eta,\quad \eta\sim Lap(2/\epsilon_L)$\;
	Record $\hat{(n_k^t)}_i$\;
	\For {$w\in D$}{
		Sample topic: $k_i\sim \mathbf{P}$\;
		\If{$w=w_{mn}$}{
			Record $k_i$\;}
	}

	$i \leftarrow i+1$\;
}
Compute $\bar{t}=\mathop{\arg\max}\limits_{t}{P[w_{mn}=t|k_1,...,k_i,...,k_T]}$ acc. to Equation~(\ref{Equ:Attack equation})\;
\end{algorithm}

For both LDP solutions \textsf{LP-LDA} and \textsf{OLP-LDA}, all documents in the training datasets $D_{\text{train}}$ were perturbed with the randomized response technique described in Section~\ref{sec:localper} in a centralized manner to simulate the procedures of crowdsourcing users. Then, for \textsf{LP-LDA}, the reconstruction process and model training were implemented. For \textsf{OLP-LDA}, the prior dataset owned by the server was generated by randomly sampling from $D_{\text{train}}$ and used for constructing the initial model. The arriving mini-batches were generated by randomly drawing from the sanitized dataset, and then reconstructed according to the Bayesian denoising scheme. Finally, \textsf{OLP-LDA} was trained on those reconstructed mini-batches.

To evaluate the practical performance on privacy protection, the membership inference attack (MIA)~\cite{shokri2017membership} was implemented to simulate the inference on trained models. Since MIA works on supervised learning models, LDA models for unlabelled datasets cannot be directly verified. Instead, the privately trained LDA models were incorporated into a classifier to derive an LDA-based classification model, which then acts as the target model for inference. In the simulation, the labeled dataset \textsf{FVENA} was used to train and test the target models. The detailed MIA attack process can be referred to~\cite{shokri2017membership,rahman2018membership}.

All our experiments were run on a laboratory-based workstation equipped with 10 cores of Intel(R) Xeon(R) E5-2640 v4@2.40GHz and 30GB memory. And the proposed algorithms were implemented with Python (version 3.7). In our experiments, for all datasets, the topic number was set as $50$, and the default hyperparameters $\alpha$, $\beta$ were set as $1$, and $0.01$, respectively.

\textbf{Metrics:}
We select \textbf{Perplexity} as the metric of LDA utility. Perplexity measures the likelihood that the test data is generated by the trained LDA model. A lower perplexity means a higher likelihood, and hence better model utility. Given a test set $D_{\text{test}}$ with $M$ documents, denote $\phi_k^{t}$ as the learned parameters from $D_{\text{train}}$ and $\theta_{m}^{k}$ as the inferred parameters from $D_{\text{test}}$, the perplexity on $D_{\text{test}}$ can be computed as
\begin{equation*}
\mathrm{per}(D_{\text{test}})=\mathrm{exp}({-\frac{\sum_{m}\sum_{i}\log{(\sum_{k}\theta_m^k\phi_k^{w_i})}}{\sum_{m}|d_m|}})
\end{equation*}
where $|d_m|$ and $w_i$ denote the number of words and the $i$-th word in $d_m$ respectively.

\textbf{Comparison:} To validate our proposed algorithm, we also compare with two algorithms:
\begin{itemize}
\item
\textsf{CDP-LDA}~\cite{foulds2016theory}, in which DP is achieved by perturbing the word count matrices $N_k^t$ and $N_m^k$ with Laplace noise $Lap(1/\epsilon)$ in the first iteration.
\item
\textsf{CDP-LDA+}, which is an extended version of \textsf{CDP-LDA}, introduces Laplace noise into $N_k^t$ and $N_m^k$ in each iteration to protect the training process.
\end{itemize}

\subsection{Performance of \textsf{HDP-LDA}}

\begin{figure*}[htb]
\centering
\subfigure[KOS]{
	\includegraphics[width=5.5cm]{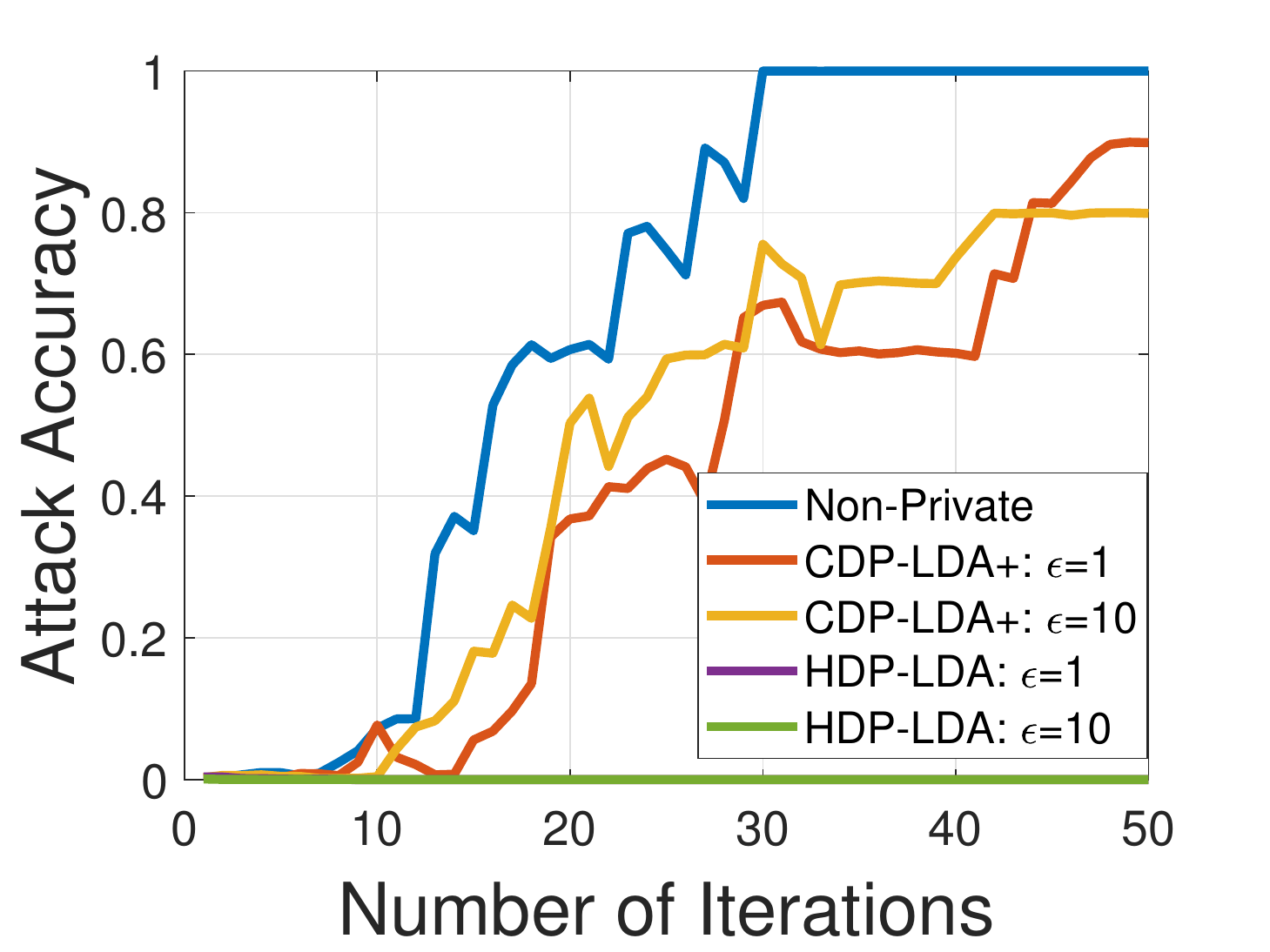}
    \label{subfig:KOS}
}
\subfigure[NIPS]{
	\includegraphics[width=5.5cm]{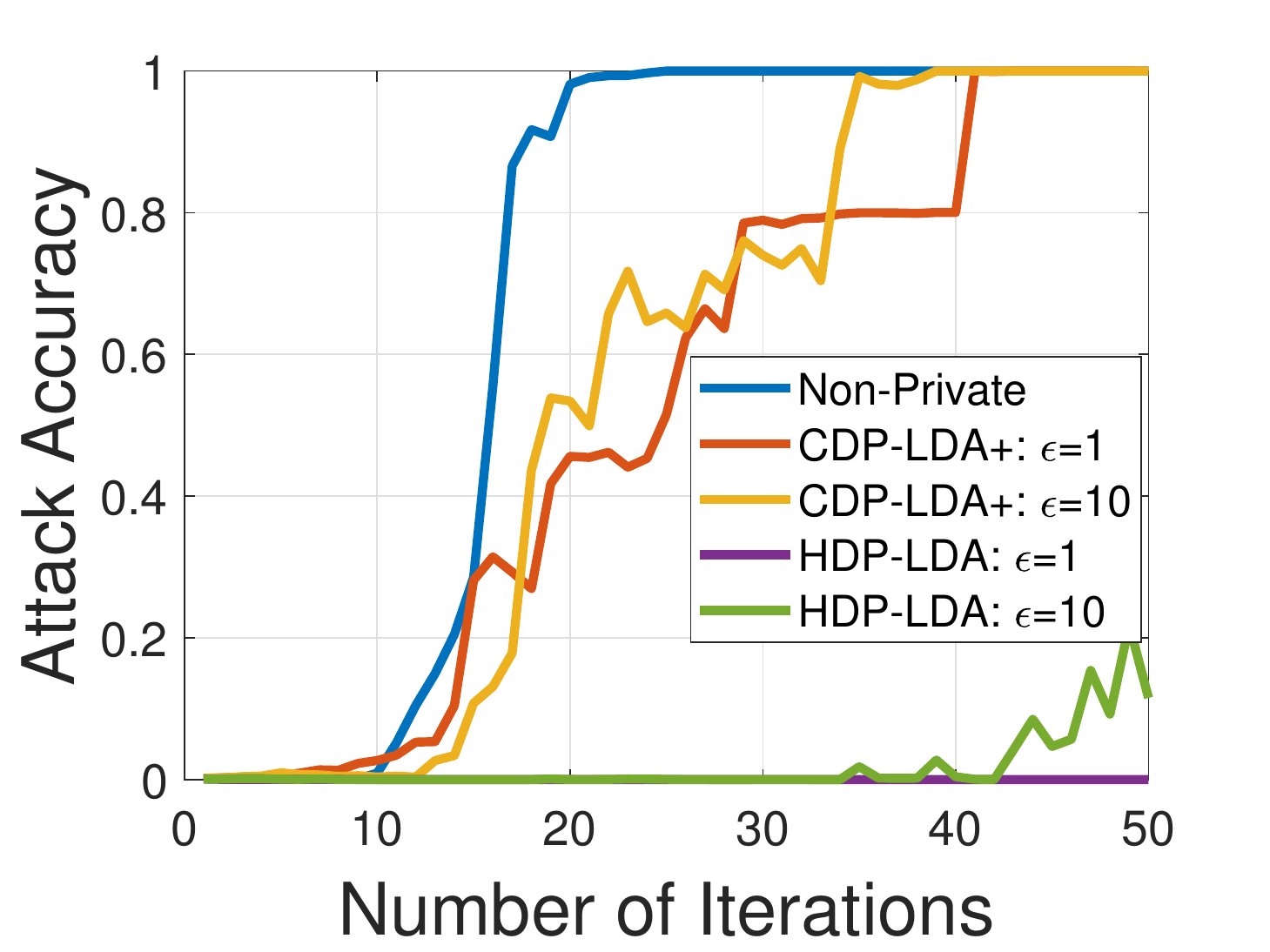}
	\label{subfig:NIPS}
}
\subfigure[Enron]{
	\includegraphics[width=5.5cm]{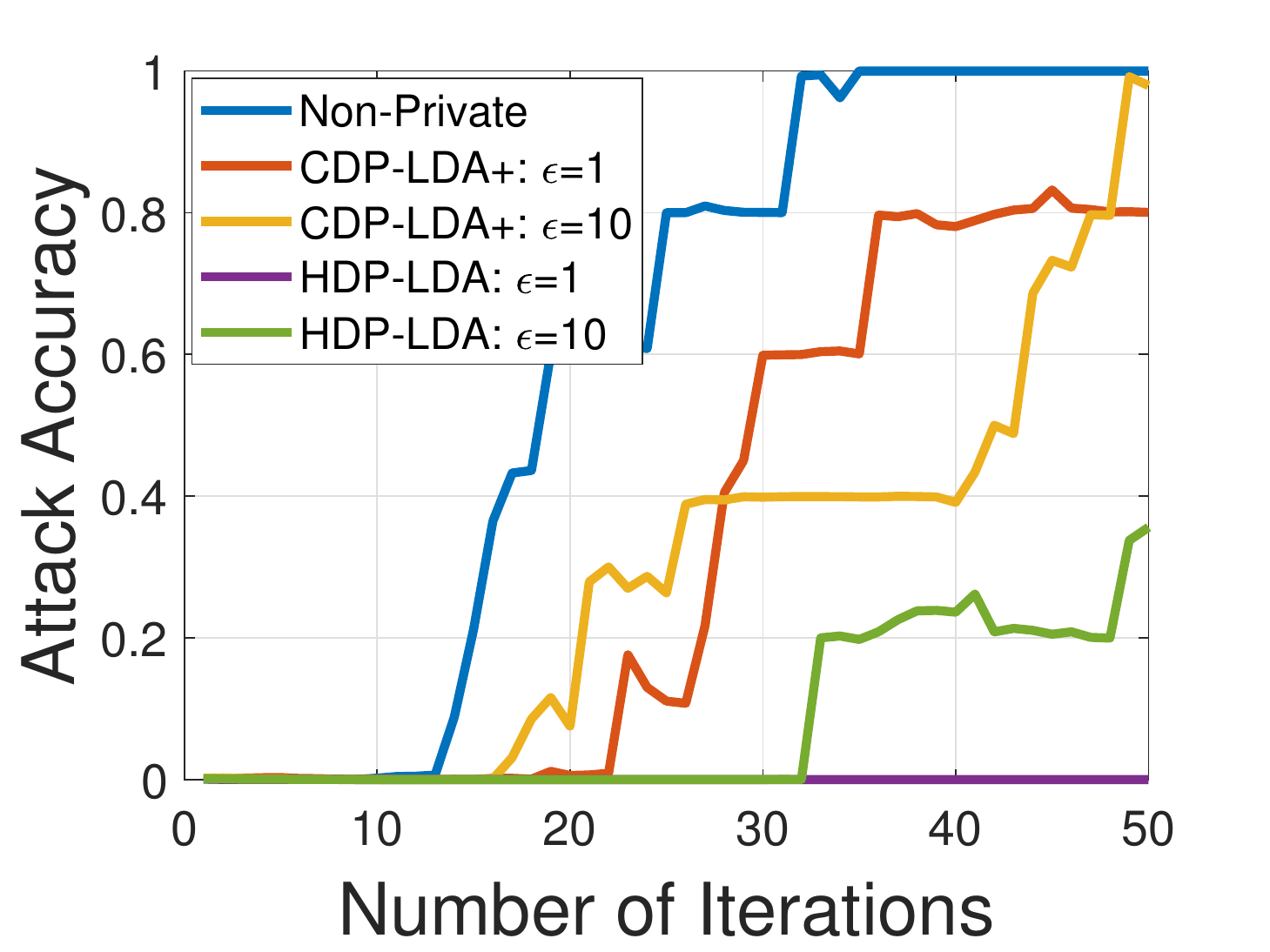}
	\label{subfig:Enron}
}
\caption{Defense Against Topic-based Attack}
\label{fig:inherent privacy}
\end{figure*}

\begin{figure*}[htb]
\centering
\subfigure[KOS]{
	\includegraphics[width=5.5cm]{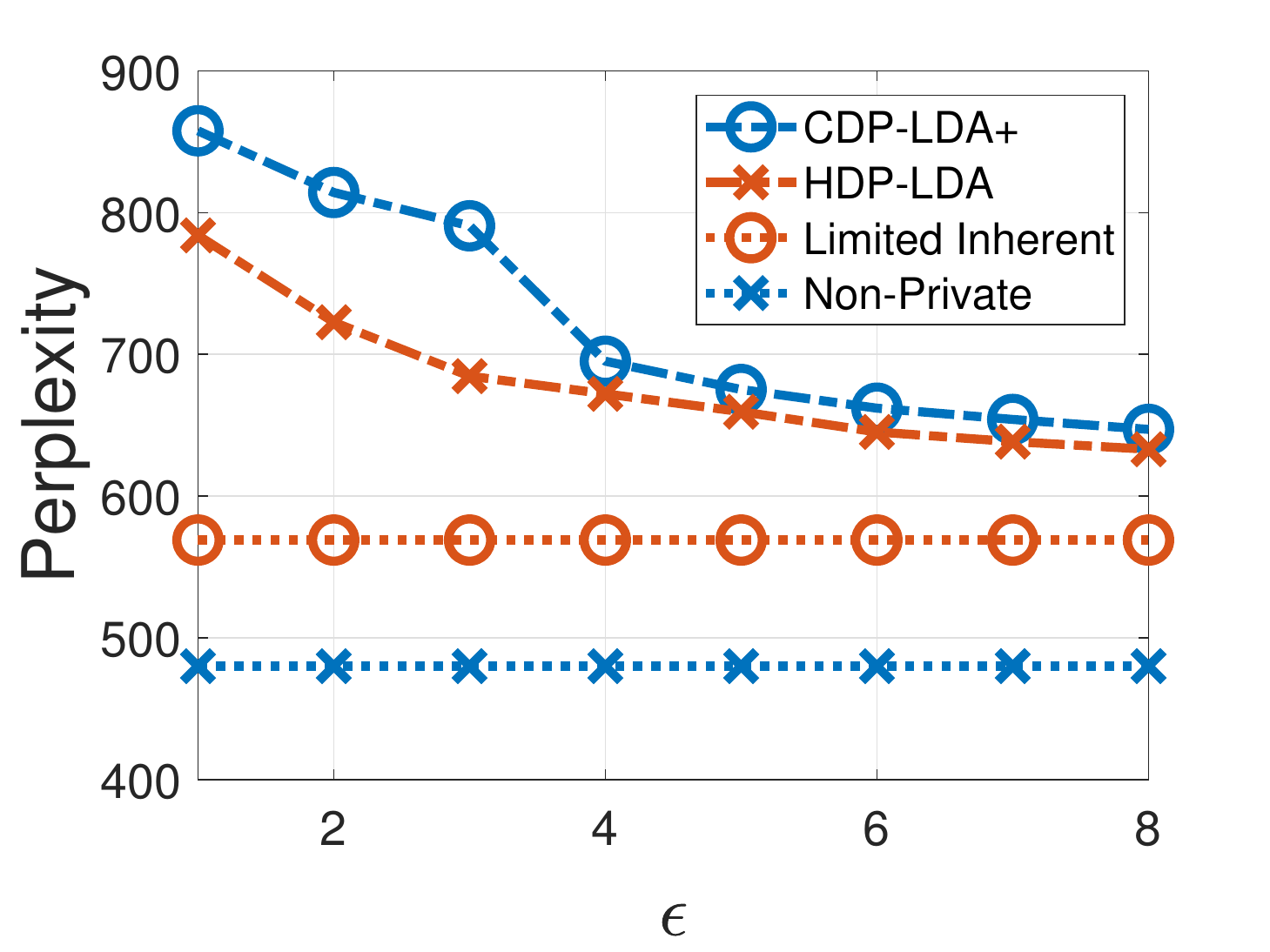}
    \label{subfig:KOS_}
}
\subfigure[NIPS]{
	\includegraphics[width=5.5cm]{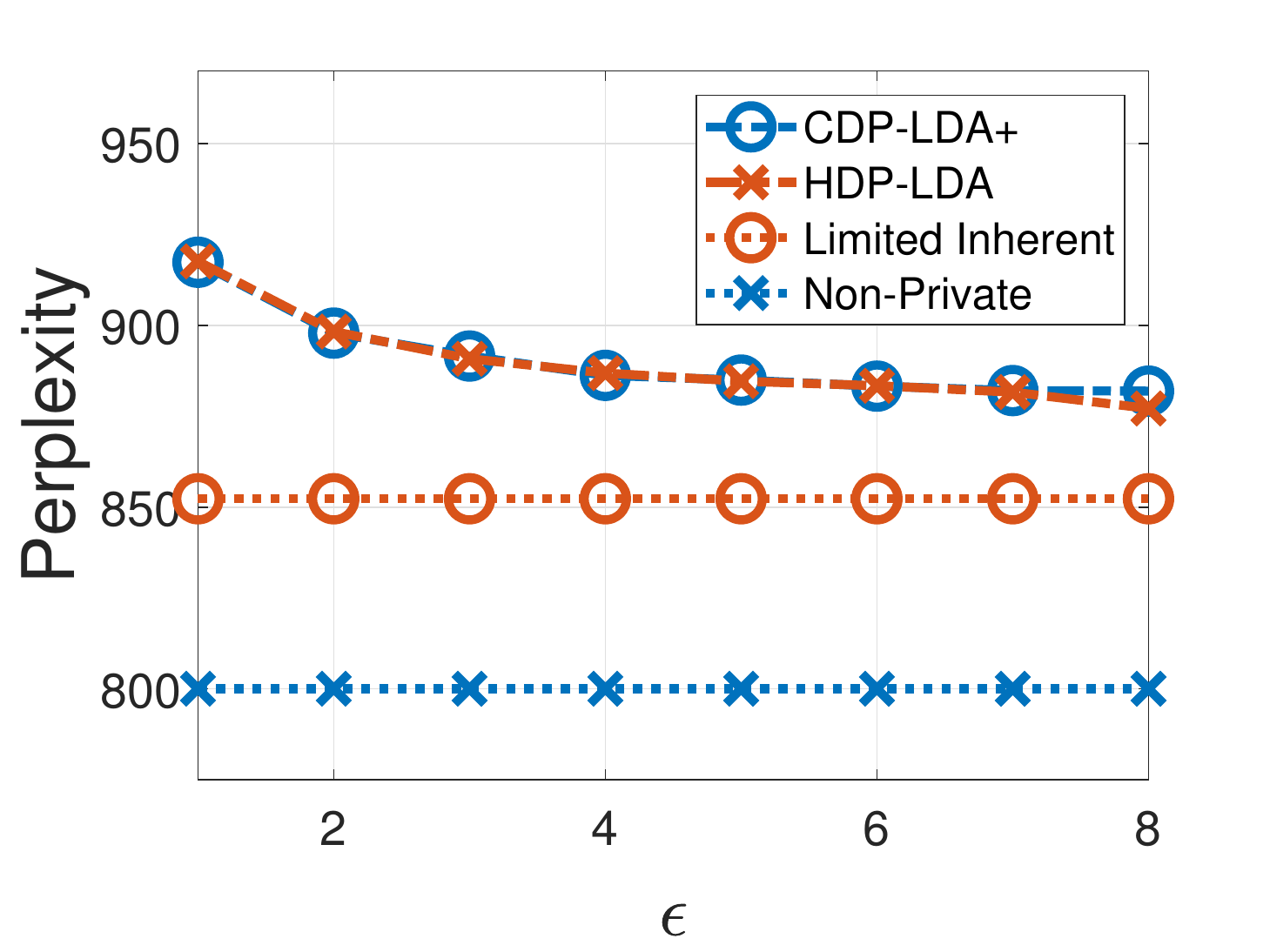}
	\label{subfig:NIPS_}
}
\subfigure[Enron]{
	\includegraphics[width=5.5cm]{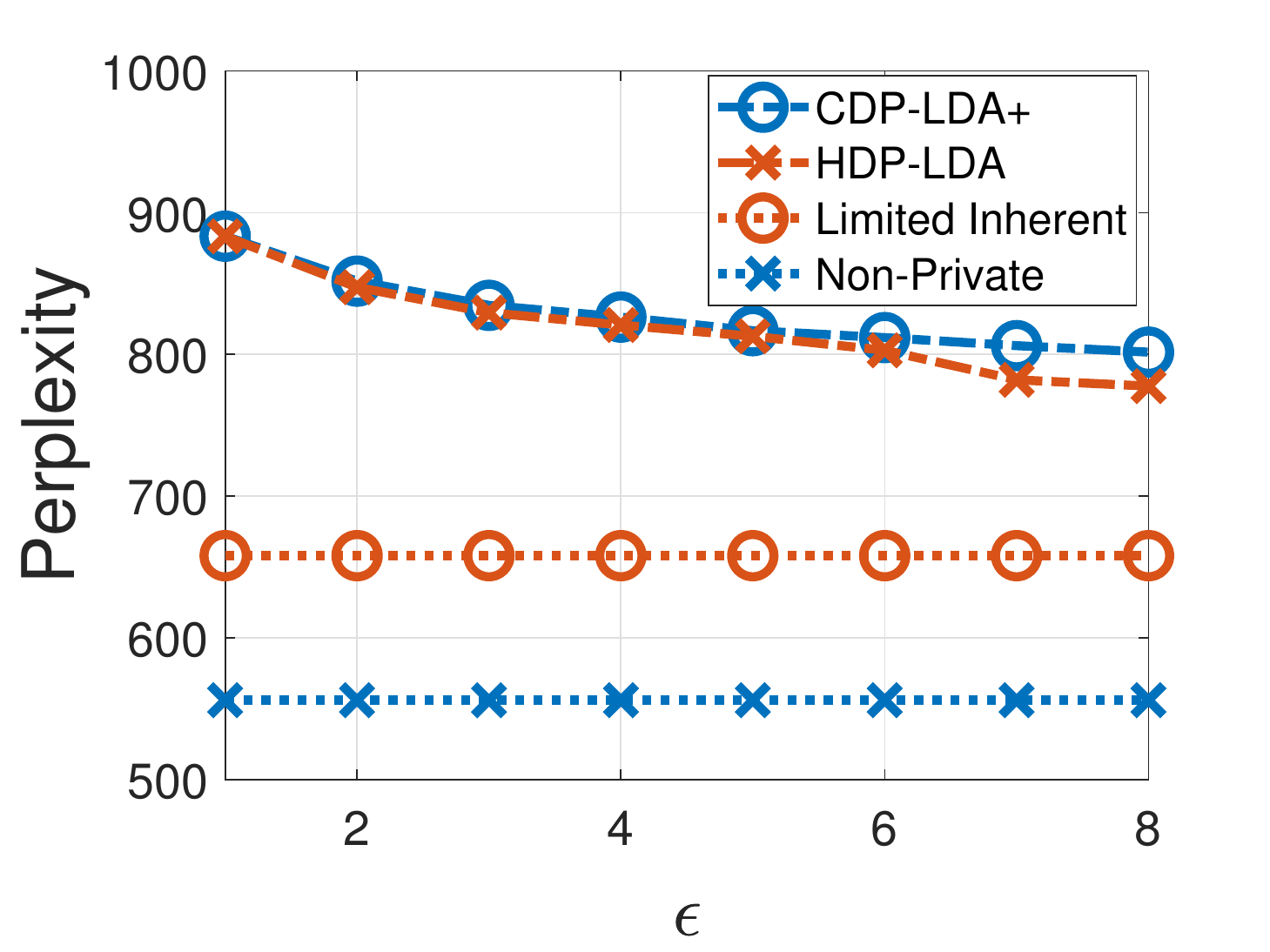}
	\label{subfig:Enron_}
}
\caption{Perplexity vs. Privacy level of \textsf{HDP-LDA}}
\label{fig:HDP-LDA vs.CDP-LDA+}
\end{figure*}

We first validate the performance of our proposed \textsf{HDP-LDA} algorithm in terms of both privacy protection and model utility. 

\subsubsection{Defense Against Topic-based Attack}
Fig.~\ref{fig:inherent privacy} reports the privacy protection of \textsf{HDP-LDA} by validating its defending ability against the \textit{Topic-based attack}. In Fig.~\ref{fig:inherent privacy}, \textsf{HDP-LDA} achieves DP by setting proper clipping bound $C$ and $\beta$ to limit the inherent privacy, according to Equation~(\ref{Equ:chain effect mitigation}). While \textsf{CDP-LDA+} achieves DP by adding Laplace noise. As shown, for a plain CGS algorithm without any intervention (referred to as \textsf{Non-Private} in Fig.~\ref{fig:inherent privacy}) and \textsf{CDP-LDA+}, the attack accuracy curves show sharp increases with the iteration number even in the strong privacy regime ($\epsilon=1$) of \textsf{CDP-LDA+}. This is because both \textsf{CDP-LDA+} and \textsf{Non-Private} can not limit the inherent privacy loss in the topic sampling process, which can accumulate rapidly with the iteration number. In contrast, we can see that \textsf{HDP-LDA} can effectively defend against the attack even when $\epsilon=10$. This demonstrates that \textsf{HDP-LDA} which limits the inherent privacy can effectively prevent the privacy leakage from the sampled topics.

\subsubsection{Utility vs. DP}
Fig.~\ref{fig:HDP-LDA vs.CDP-LDA+} compares the perplexity of \textsf{HDP-LDA} and \textsf{CDP-LDA+} under different privacy levels $\epsilon$ incurred by the Laplace noise. In Fig.~\ref{fig:HDP-LDA vs.CDP-LDA+}, \textsf{HDP-LDA} limits the inherent privacy level as $10$ in each iteration under which the topic sampling process can be protected as shown in Fig.~\ref{fig:inherent privacy}. The curves marked by \textsf{Limited Inherent} present the perplexity of \textsf{HDP-LDA} with limited inherent privacy as $10$ but no Laplace noise. Compared with \textsf{Non Private} (or plain CGS), \textsf{Limited Inherent} shows a utility degradation since it has stronger inherent privacy (through clipping the word-counts and setting the hyper-parameter $\beta$ larger).

However, the model utility of \textsf{HDP-LDA} is no worse than \textsf{CDP-LDA+}, and much better on \textsf{KOS} dataset, even if \textsf{CDP-LDA+} incurs more privacy loss than \textsf{HDP-LDA} when taking the inherent privacy loss into account. That is because in \textsf{HDP-LDA}, we set a larger $\beta$ in the training process as the prior information which could improve the model robustness to the noise.

\subsection{Performance of \textsf{LP-LDA}}
\begin{figure*}[htb]
\centering
\subfigure[\textsf{KOS}]{
	\includegraphics[width=5.5cm]{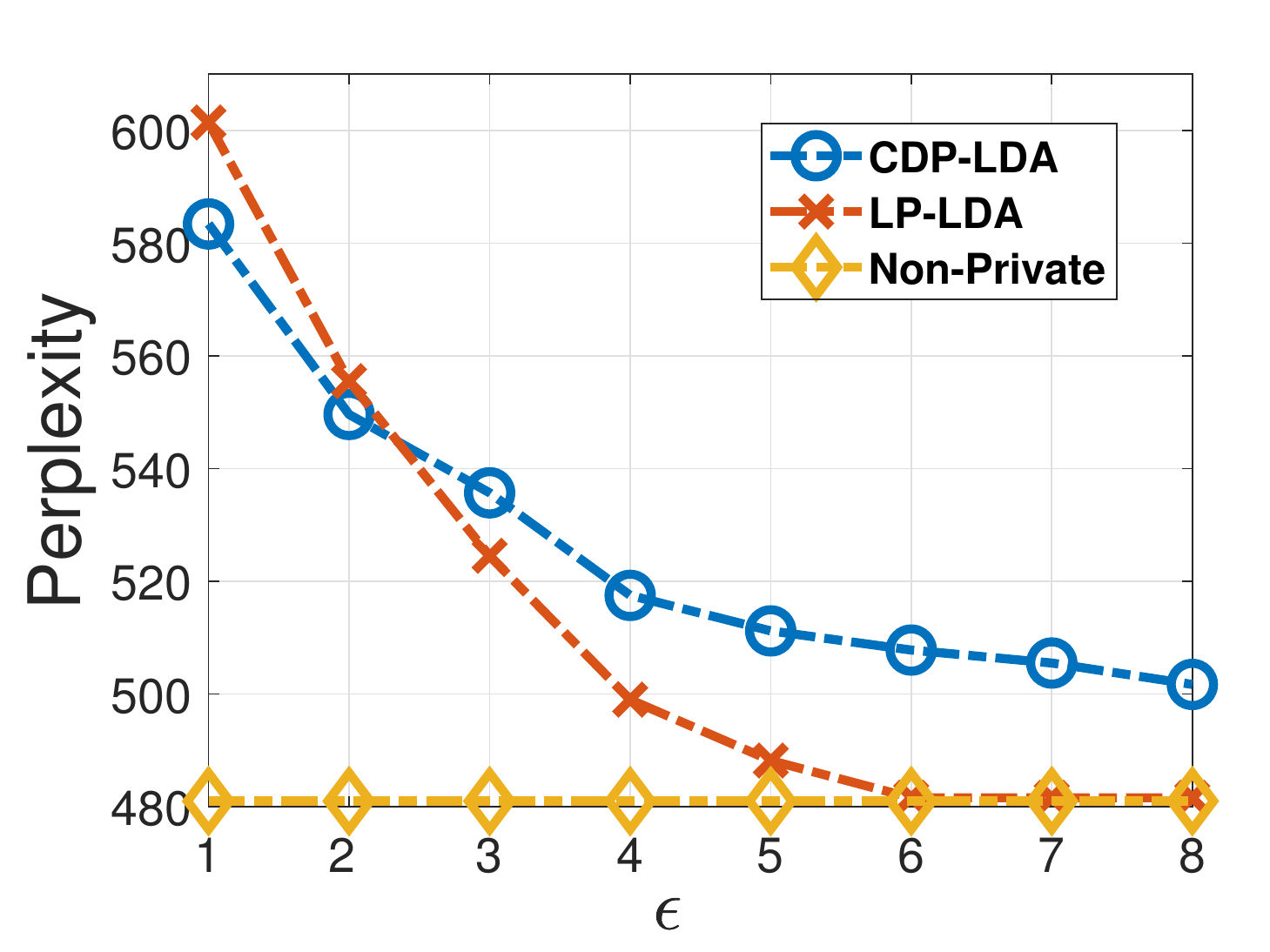}
}
\subfigure[\textsf{NIPS}]{
	\includegraphics[width=5.5cm]{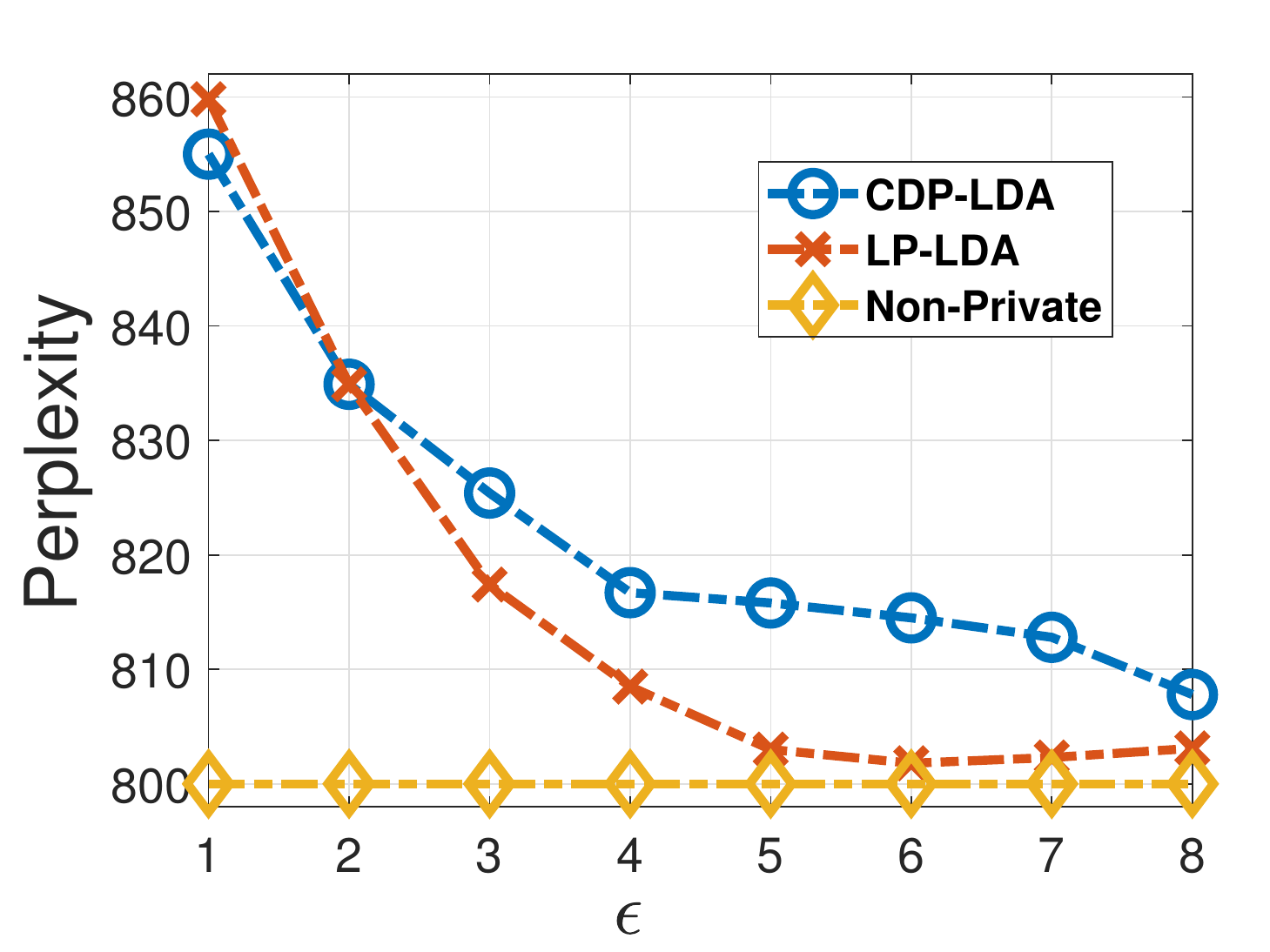}
}
\subfigure[\textsf{Enron}]{
	\includegraphics[width=5.5cm]{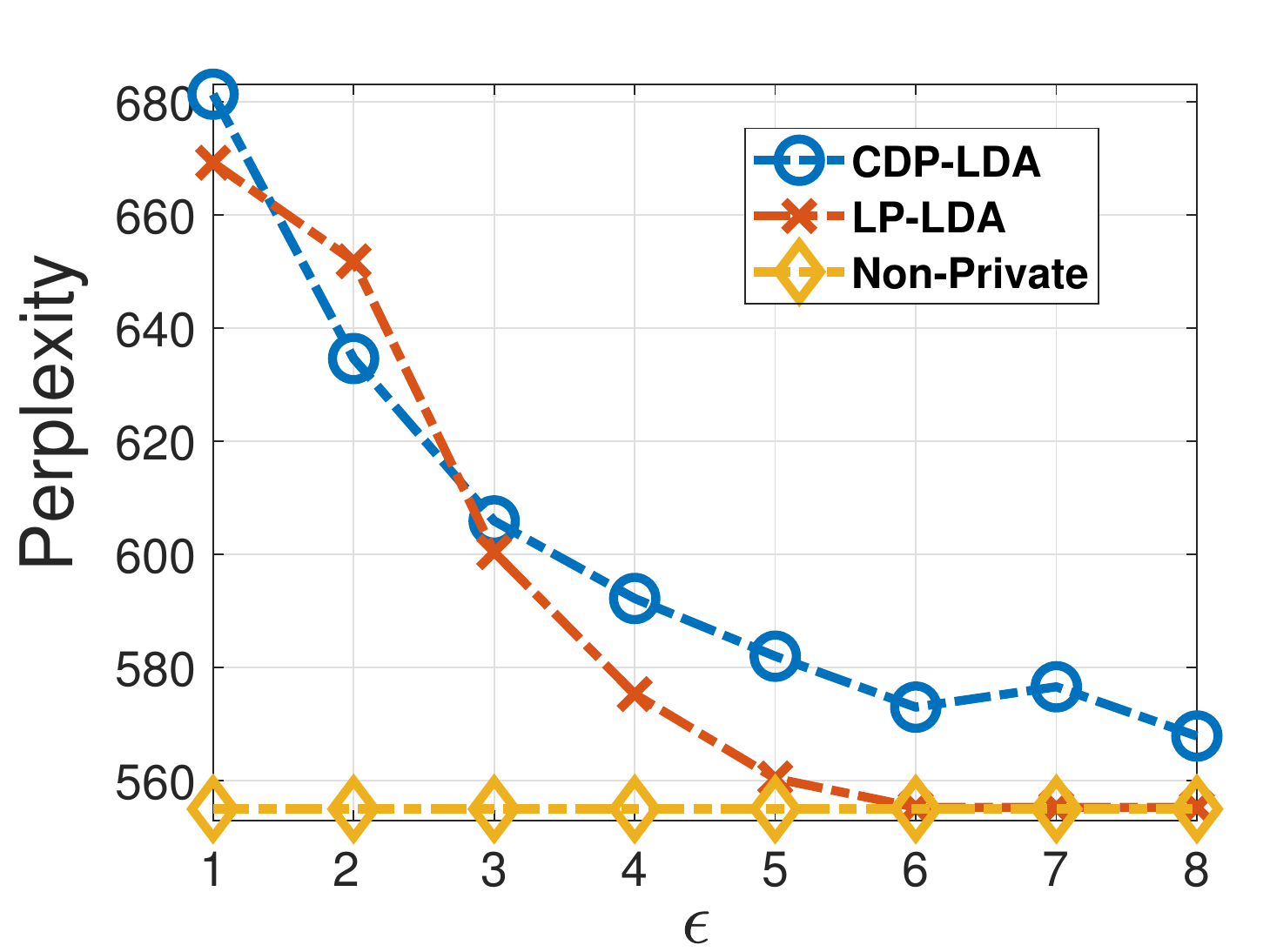}
}
\caption{Perplexity vs. Privacy Level of \textsf{LP-LDA}}
\label{fig:LP-LDA}
\end{figure*}

We then present the performance of our proposed LDP based LDA algorithm \textsf{LP-LDA} for static dataset.
\subsubsection{Utility vs. DP}
Fig.~\ref{fig:LP-LDA} shows the perplexity of \textsf{LP-LDA} in comparison with \textsf{CDP-LDA} under different levels of local differential privacy. As we can see, the perplexity curves of both \textsf{LP-LDA} and \textsf{CDP-LDA} show monotonous decrease as $\epsilon$ increases, which illustrates the tradeoff between the model utility and privacy. In particular, for stronger privacy regime with smaller $\epsilon$, the compared \textsf{CDP-LDA} performs better than \textsf{LP-LDA} because the injected Laplace noise incurs less randomness than the randomized response technique of \textsf{LP-LDA}. However, for weaker privacy regime with larger $\epsilon$, \textsf{LP-LDA} performs much better than \textsf{CDP-LDA}, even close to the non-private LDA model as $\epsilon$ keeps increasing.

\subsubsection{Defense Against MIA}
We also present the privacy guarantee of \textsf{LP-LDA} by verifying the defending ability of its LDA-based classification model against MIA. In particular, we compare the membership inference accuracy between \textsf{LP-LDA} and \textsf{CDP-LDA}. The baseline non-private model was implemented by performing the CGS training process for $300$ iterations. For reference, we also present the performance of the target LDA-based classification models.
\begin{figure*}[htp]
\centering
\subfigure[\textsf{Attack Accuracy}]{
	\includegraphics[width=5.5cm]{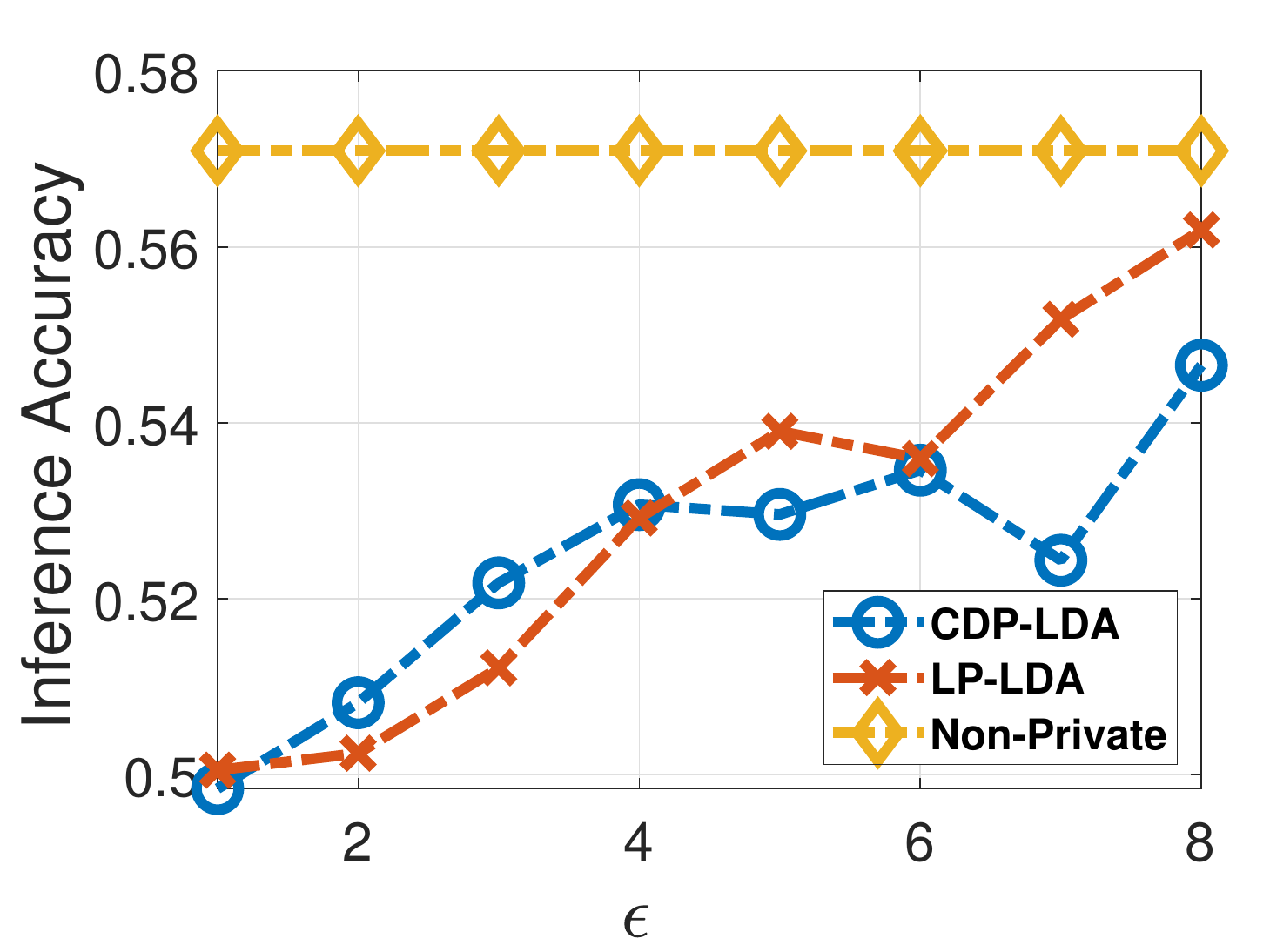}
	\label{attack_accuracy}
}
\subfigure[\textsf{Prediction Accuracy}]{
	\includegraphics[width=5.5cm]{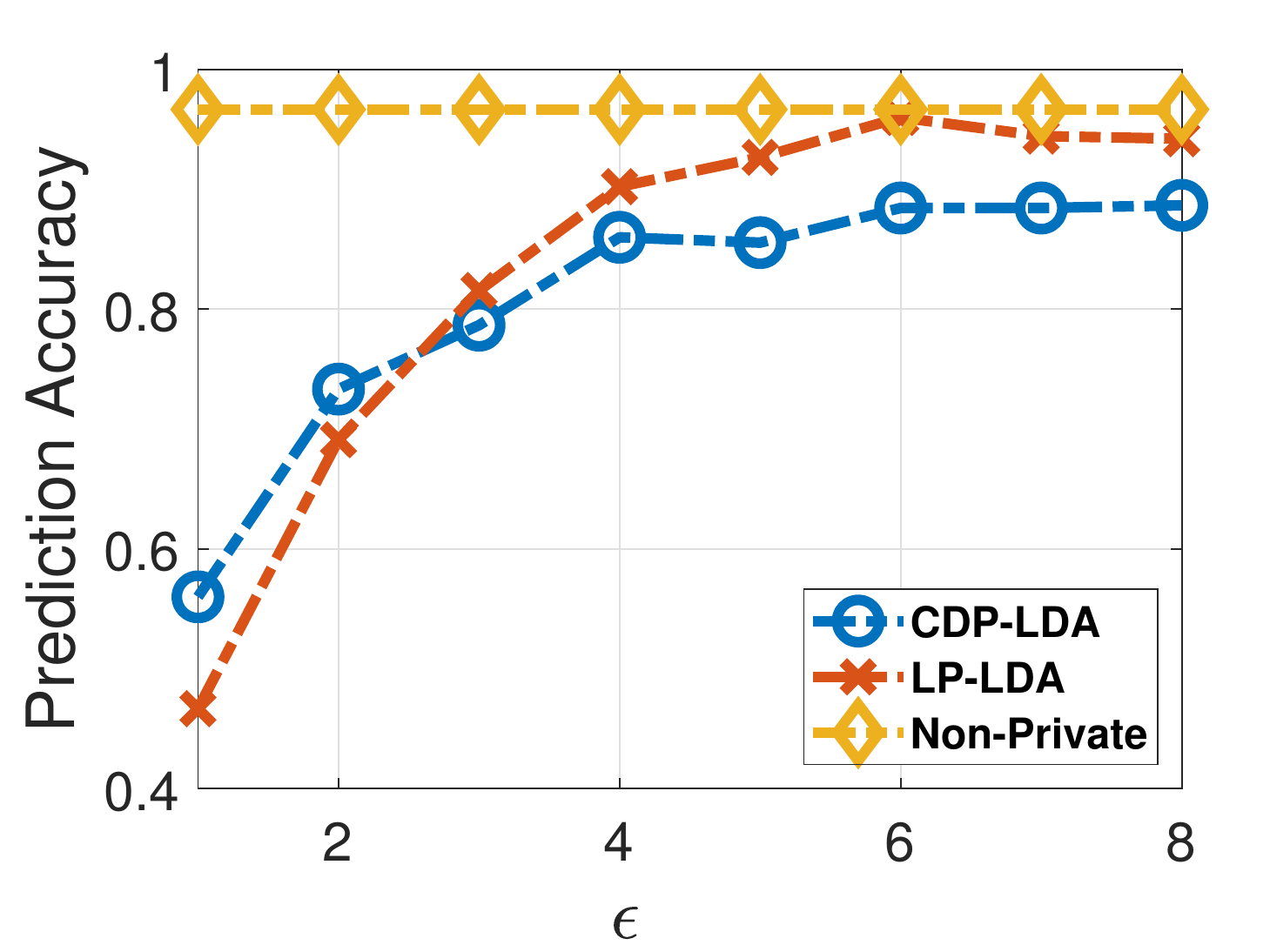}
	\label{model_prediction_accuracy}
}
\subfigure[\textsf{Prediction F1-Score}]{
	\includegraphics[width=5.5cm]{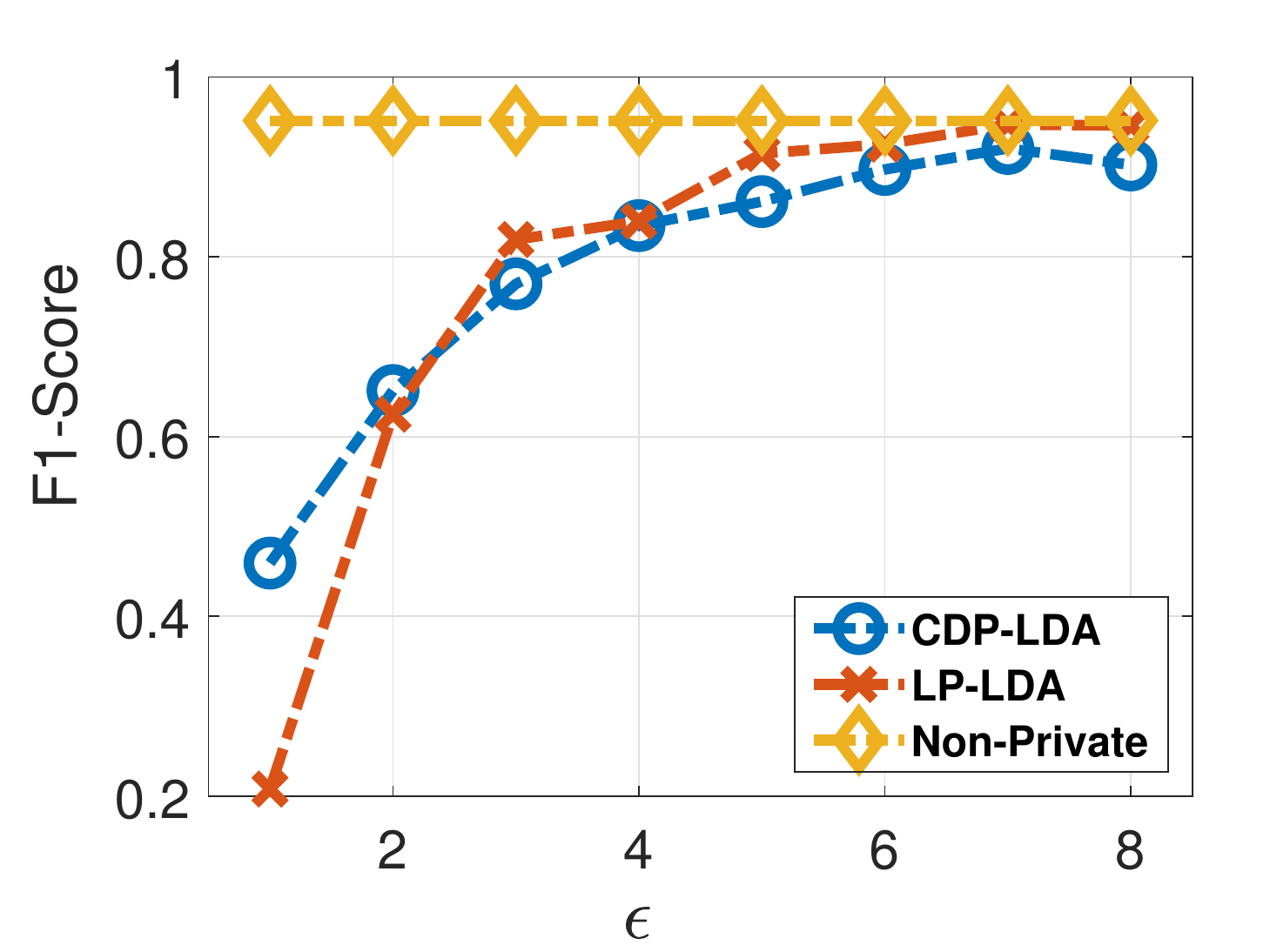}
	\label{model_F1_score}
}
\caption{Defense Against Membership Inference Attack}
\label{fig:MIA}
\end{figure*}	

Fig.~\ref{attack_accuracy} shows the membership inference accuracy of MIA on the LDA-based classification model generated by \textsf{LP-LDA}.
As shown, for both \textsf{LP-LDA} and \textsf{CDP-LDA} methods, the inference accuracy curves are below that of the CGS algorithm, which demonstrates that both privacy-preserving methods can effectively mitigate the inference of MIA. Furthermore, both inference accuracy curves show almost a monotonous increasing trend with the increase of $\epsilon$, which demonstrates that smaller $\epsilon$-DP could provide stronger defense ability.

Fig.~\ref{model_prediction_accuracy} and Fig.~\ref{model_F1_score} present the model utility (i.e., prediction accuracy and F1 score) of the LDA-based classification model generated by \textsf{LP-LDA}, versus the privacy levels $\epsilon$. As shown, both the prediction accuracy and F1-score for \textsf{LP-LDA} and \textsf{CDP-LDA} increase with $\epsilon$. This illustrates the general tradeoff between the privacy level and model utility. Similar to the results shown in Fig.~\ref{fig:LP-LDA}, the LDA-based classification model of \textsf{LP-LDA} performs better than \textsf{CDP-LDA} in the weaker privacy regime. As $\epsilon$ keeps increasing, \textsf{LP-LDA} can even approach to the plain CGS algorithm without extra noise.

\subsection{Performance of \textsf{OLP-LDA}}
We finally present the performance of our proposed \textsf{OLP-LDA} algorithm. In each experiment, the correlation factor $\lambda$ was set to $0.5$, the reconstruction factor $\omega$ was set to $0.4$, the mini-batch size was set to $160$ and the sizes of  the prior dataset for \textsf{KOS}, \textsf{NIPS} and \textsf{Enron} were set to $500$, $200$, and $1, 000$ respectively.

\subsubsection{Effect of Privacy Levels}

\begin{figure*}[htb]
\centering
\subfigure[\textsf{KOS}]{
	\includegraphics[width=5.5cm]{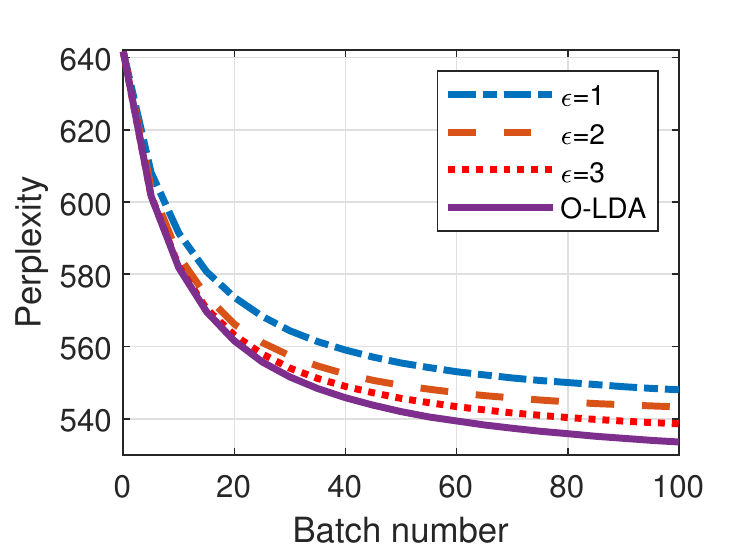}
	\label{on_kos}
}
\subfigure[\textsf{NIPS}]{
	\includegraphics[width=5.5cm]{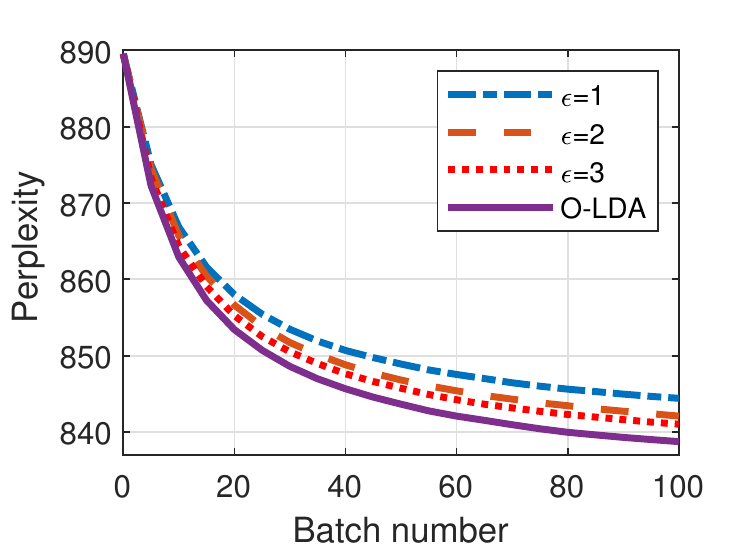}
	\label{on_nips}
}
\subfigure[\textsf{Enron}]{
	\includegraphics[width=5.5cm]{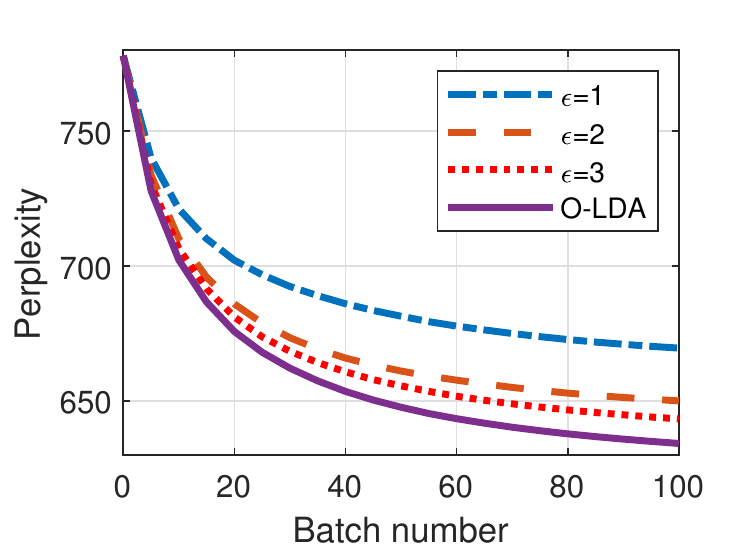}
	\label{on_enron}
}
\caption{Perplexity of OLP-LDA vs. Batch Number }
\label{fig:OLP-LDA}
\end{figure*}
Fig.~\ref{fig:OLP-LDA} shows the average perplexity (normalized by the number of mini-batches) of our proposed \textsf{OLP-LDA} versus the number of mini-batches, under different privacy levels. Particularly, the non-private solution \textsf{O-LDA} is also compared for reference. As we can see in all three subfigures, the perplexity curves for different privacy levels decrease monotonically with the increase of batch number, which demonstrates that \textsf{OLP-LDA} can iteratively refine the model with the continuous training batches. And also we can see, the different privacy level leads to different model performance in terms of perplexity, which also illustrates the tradeoff between the model performance and the privacy level.

\subsubsection{Effect of the Prior Dataset Size}
\begin{figure*}[htb]
\centering
\subfigure[\textsf{KOS}]{
	\includegraphics[width=5.5cm]{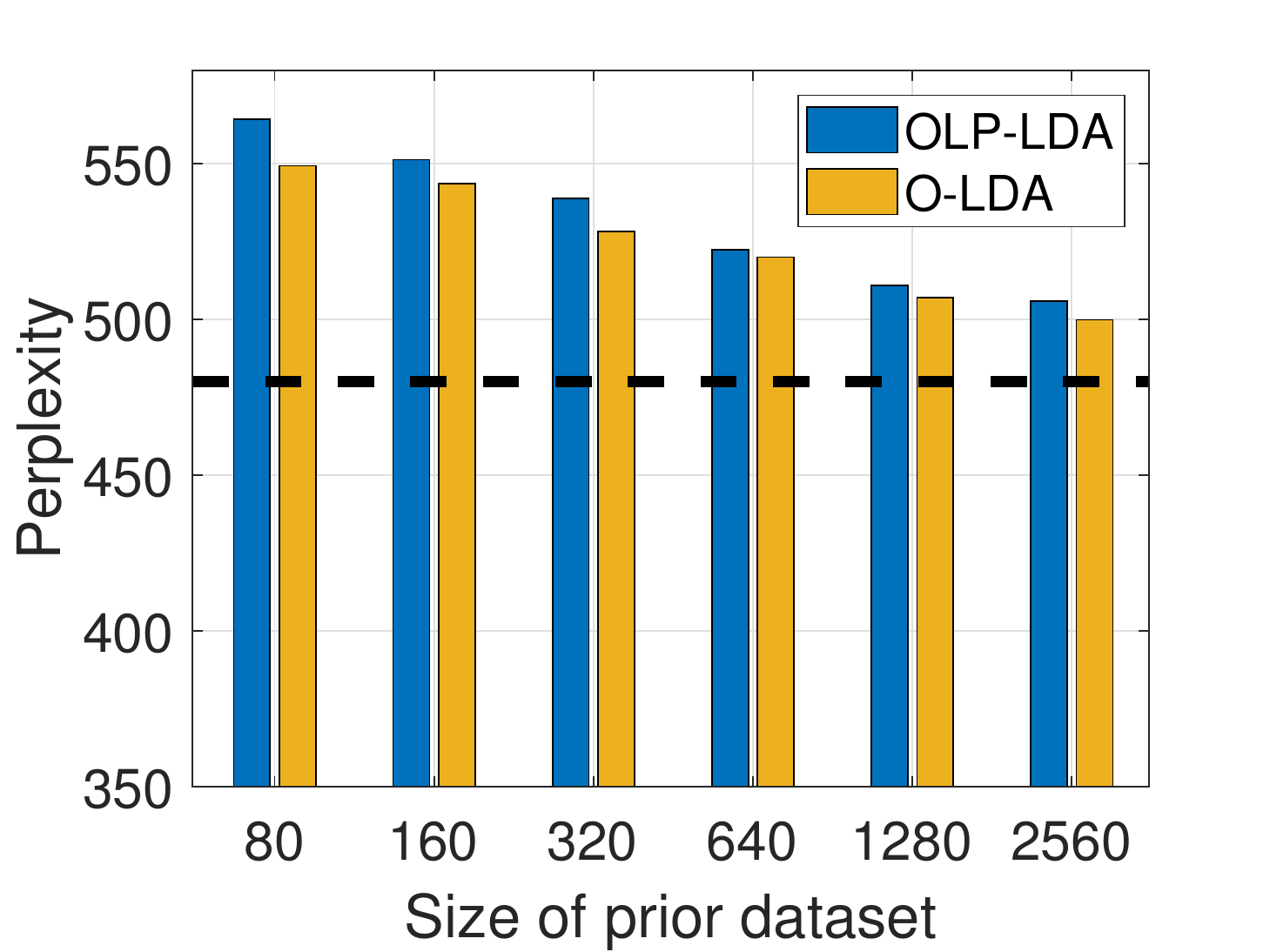}
	\label{on_nips}
}
\subfigure[\textsf{NIPS}]{
	\includegraphics[width=5.5cm]{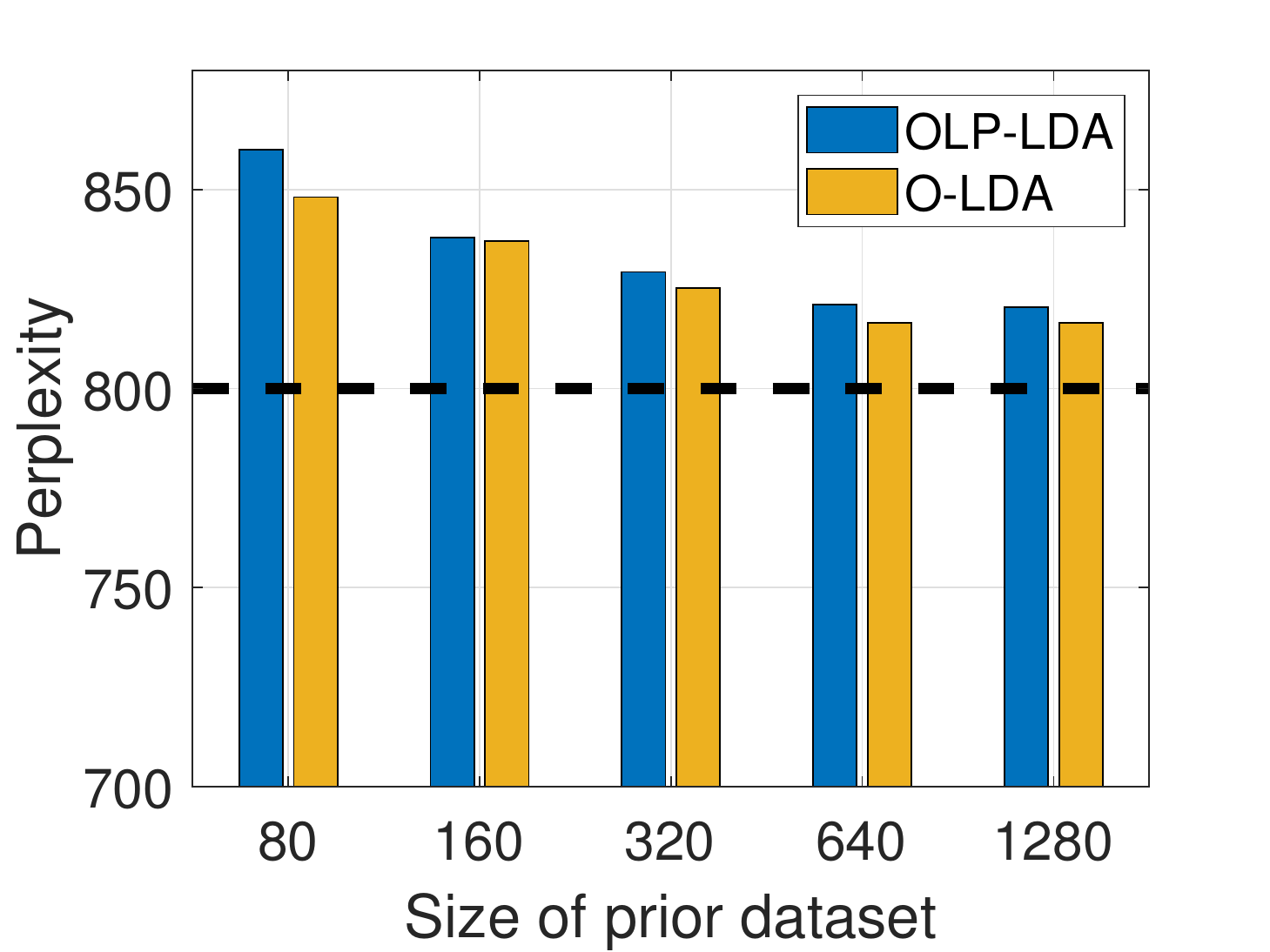}
	\label{on_nips}
}
\subfigure[\textsf{Enron}]{
	\includegraphics[width=5.5cm]{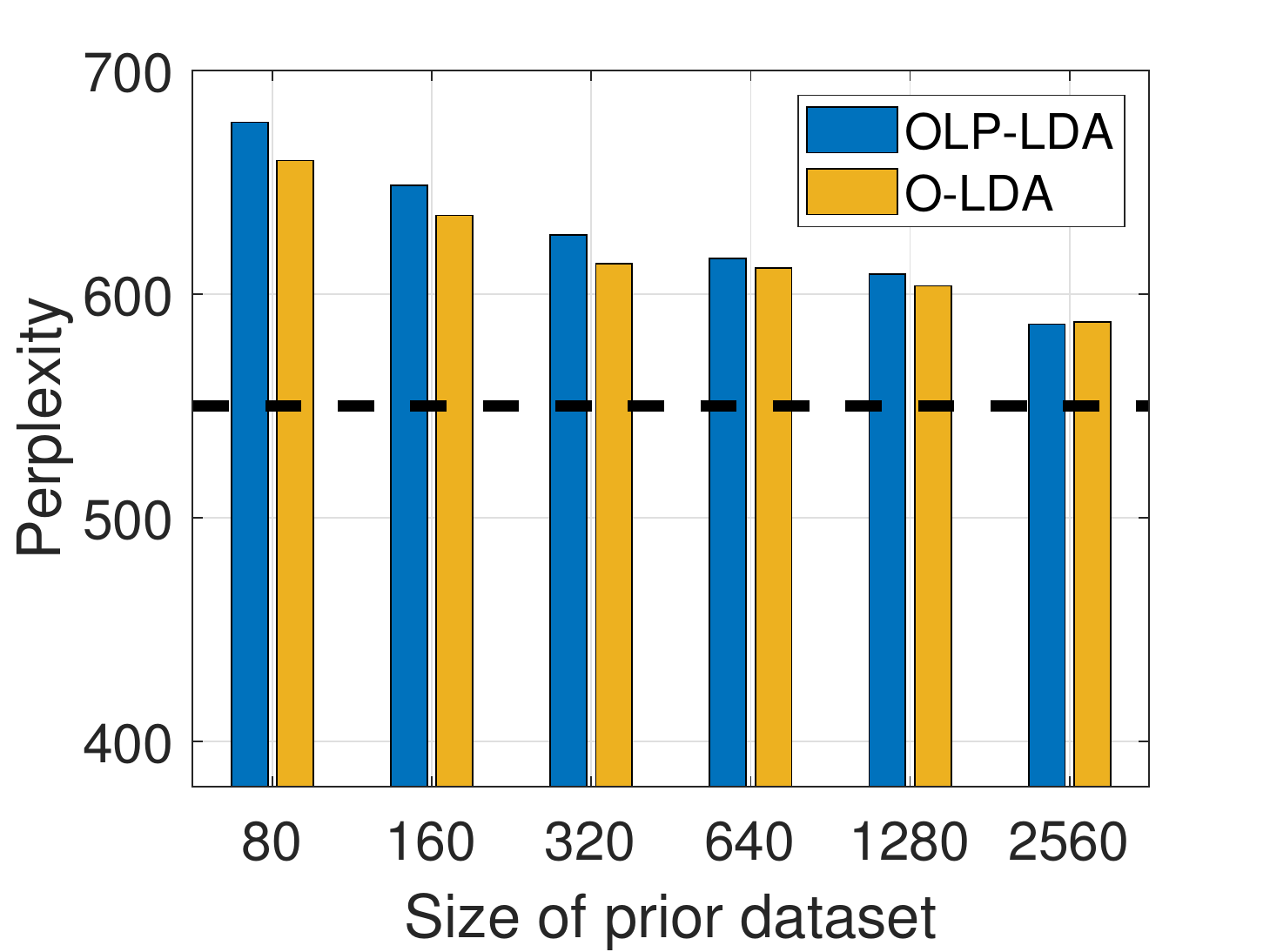}
	\label{on_enron}
}
\caption{Perplexity of \textsf{OLP-LDA} vs. Prior Data Size}
\label{fig:prior_OLP-LDA}
\end{figure*}
Fig.~\ref{fig:prior_OLP-LDA} depicts the performance of \textsf{OLP-LDA} in comparison with its non-private version \textsf{O-LDA}, under different sizes of the prior dataset. The privacy level in \textsf{OLP-LDA} was set to $3.0$ and the black horizontal dashed lines marks the perplexity obtained by the non-private batched CGS training algorithm. As expected, a larger scale of the prior dataset leads to a better model performance, that's because a larger prior dataset could provide more accurate prior knowledge for the word counts and reflect more representative information of the whole corpus.

\subsubsection{Effect of the Mini-batch Size}
\begin{figure*}[htp]
\centering
\subfigure[\textsf{KOS}]{
	\includegraphics[width=5.5cm]{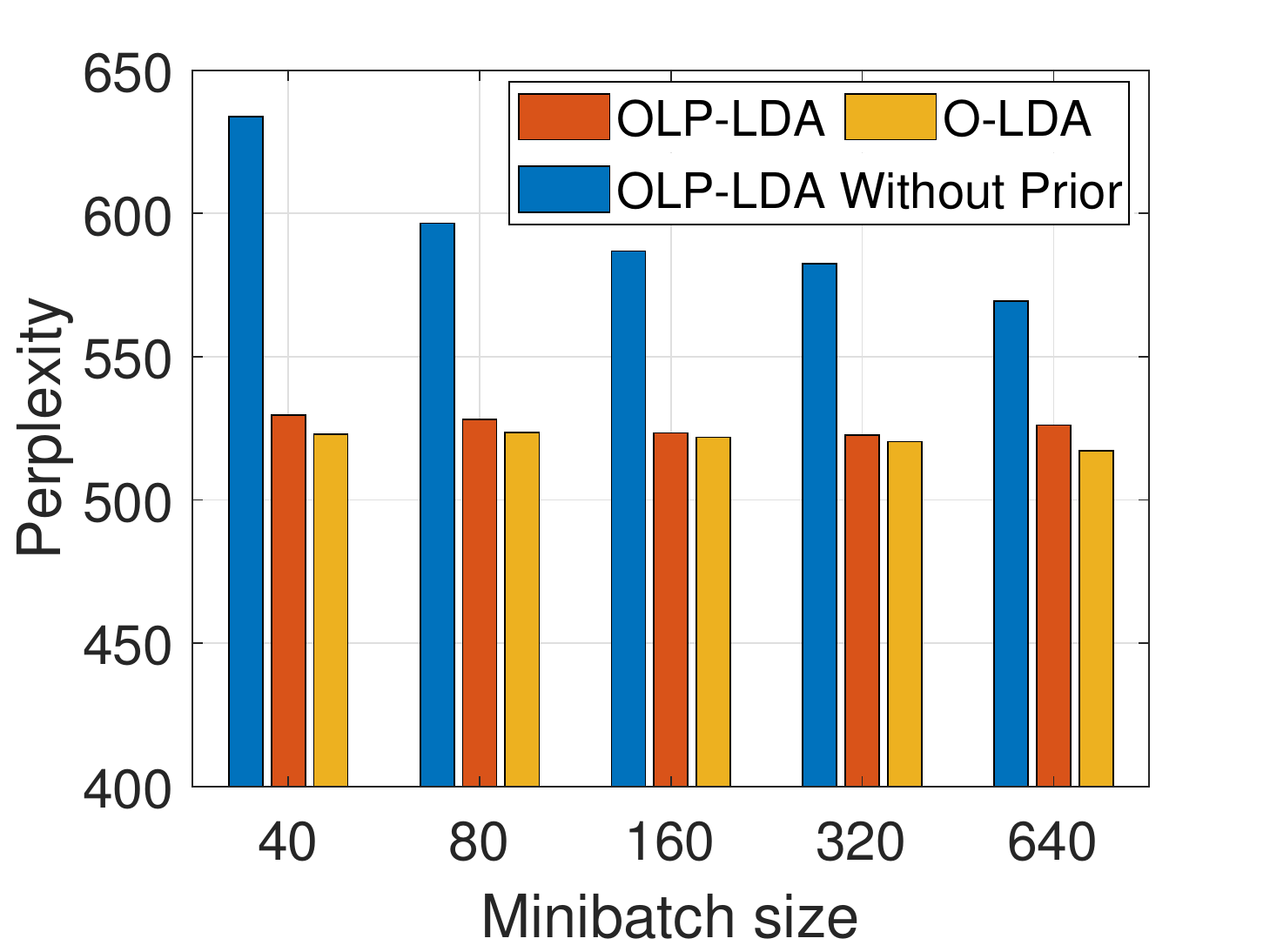}
	\label{minibatch_on_kos}
}
\subfigure[\textsf{NIPS}]{
	\includegraphics[width=5.5cm]{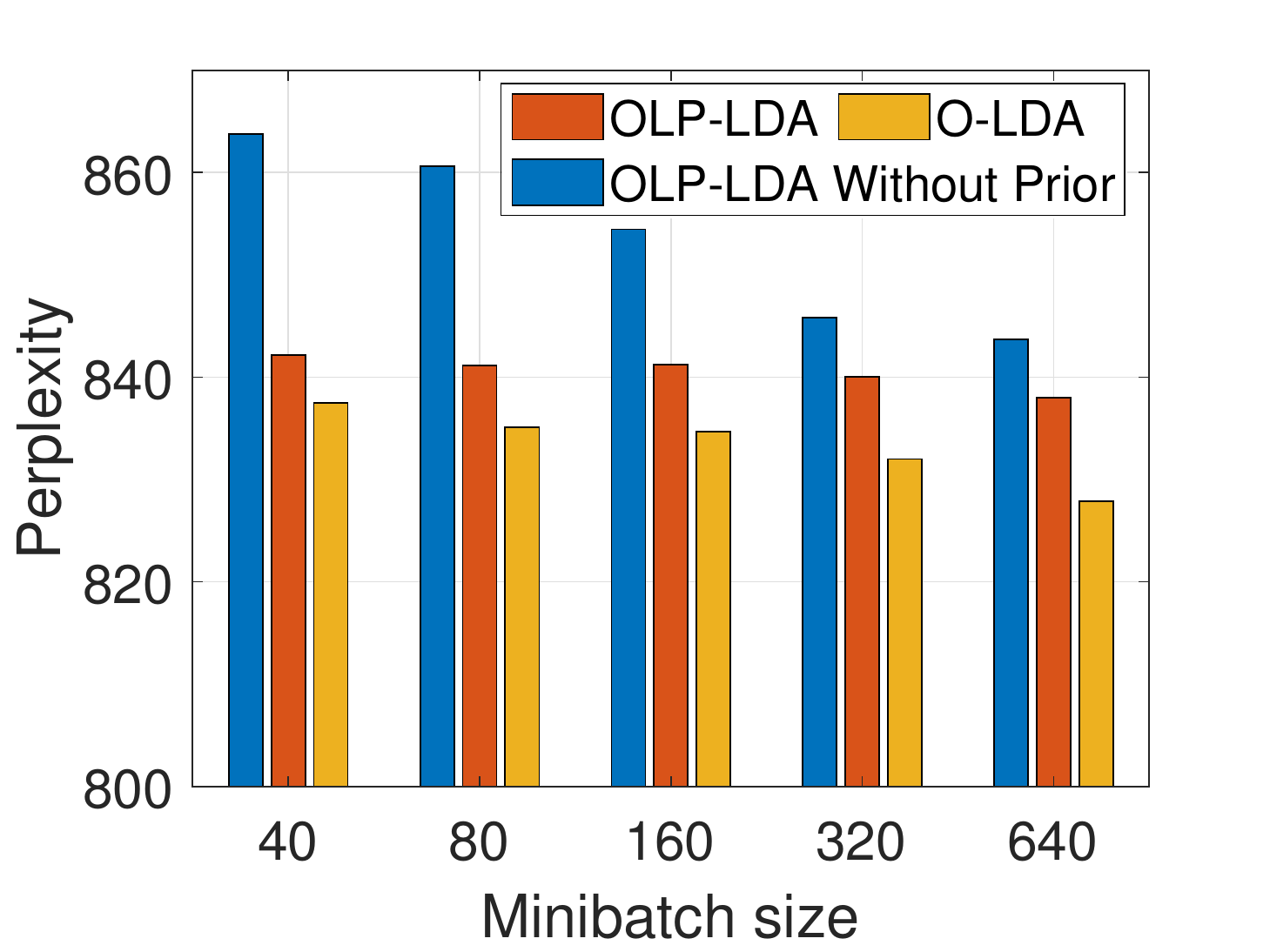}
	\label{minibatch_on_nips}
}
\subfigure[\textsf{Enron}]{
	\includegraphics[width=5.5cm]{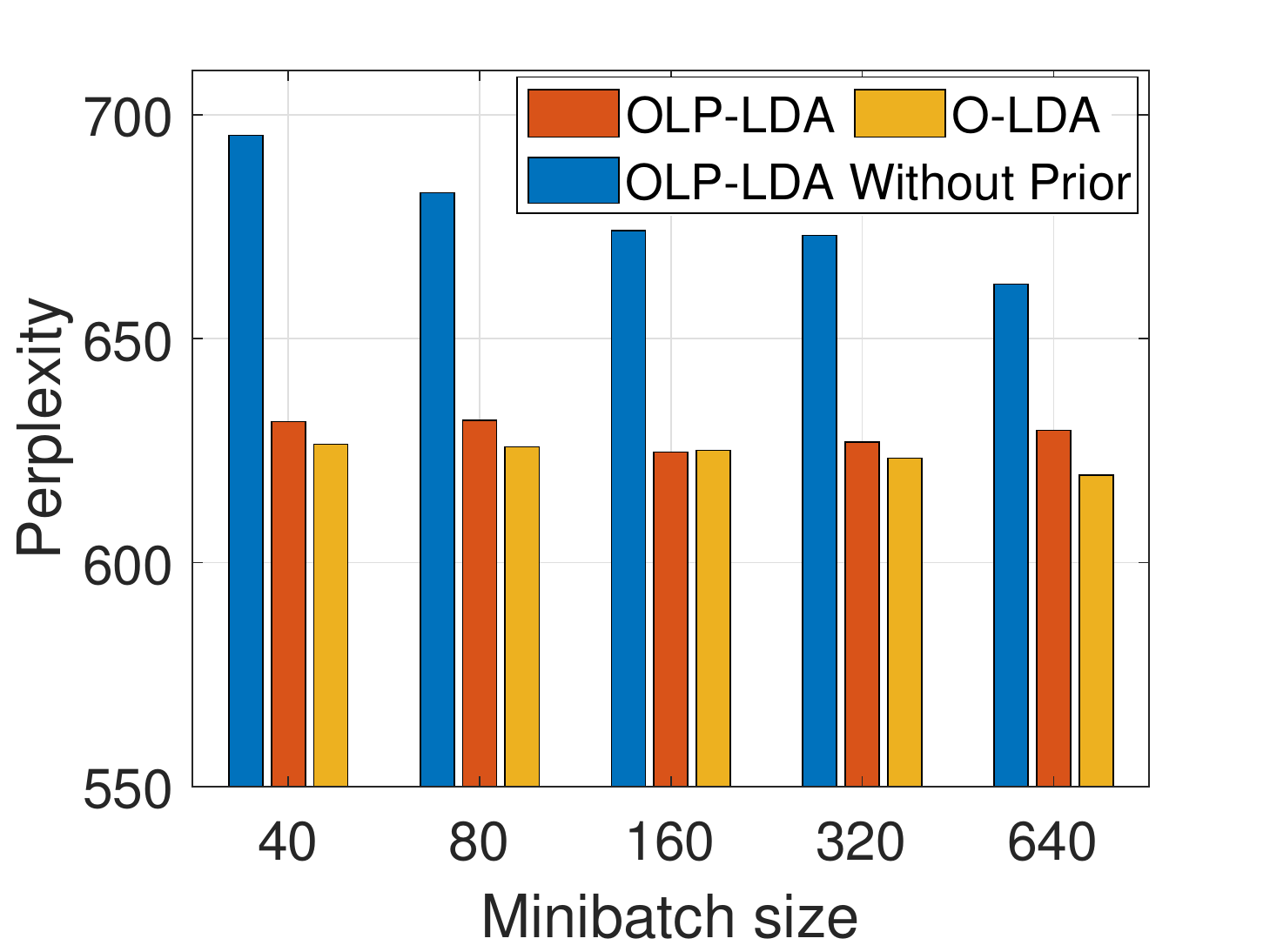}
	\label{minibatch_on_enron}
}
\caption{Perplexity of \textsf{OLP-LDA} vs. Mini-batch Size. }
\label{fig:minibatch_factor}
\end{figure*}
Fig.~\ref{fig:minibatch_factor} shows the impact of the mini-batch size on the model performance of \textsf{OLP-LDA}. For comparison, we also present the model performance of both \textsf{OLP-LDA} without the prior data and the non-private algorithm \textsf{O-LDA}.	
As shown, the perplexity of \textsf{OLP-LDA} without prior information is much larger than those of both \textsf{O-LDA} and \textsf{OLP-LDA} with prior dataset. Besides, the perplexity obtained by \textsf{OLP-LDA} with prior data can be very close to that in the non-private \textsf{O-LDA}. These observations show that the high effectiveness of our proposed \textsf{OLP-LDA} and validate that prior data can better enhance the performance of \textsf{OLP-LDA}.	
Furthermore, the perplexity of both \textsf{OLP-LDA} without prior data and \textsf{O-LDA} fall with the increase of mini-batch size respectively. However, for \textsf{OLP-LDA}, the perplexity does not show a clear decreasing trend with the growth of the mini-batch size, and even slightly increases, which means \textsf{OLP-LDA} is relatively not sensitive to the mini-batch size. This is because the prior dataset has sufficiently optimized the model of \textsf{OLP-LDA} and achieved the approximate optimal model as \textsf{O-LDA}.

\section{Conclusion}\label{sec:conclusion}
This paper investigates the privacy protection of Latent Dirichlet Allocation model training in the framework of differential privacy. We present the first analysis on the intrinsic privacy-preserving effect possessed by Collapsed Gibbs Sampling based LDA training algorithm and further propose a centralized privacy-preserving algorithm (HDP-LDA) that can prevent data inference from the intermediate statistics in the CGS training. Concerning the privacy risks at the central server, we also design an LDP solution of \textsf{LP-LDA} by sanitizing the individual documents at the local side and performing the model training on the reconstruction dataset at the server side. For LDA training in a stream setting, we further provide an online algorithm \textsf{OLP-LDA} that can efficiently refine the training model on continuously perturbed data batches. Both extensive analysis and experimental results on real-world datasets validate and demonstrate that our proposed algorithms can achieve high model utility under differential privacy.

\bibliographystyle{IEEEtran}
\small
\bibliography{references}

\appendices
\normalsize

\section{Proof of Proposition~\ref{proposition:privacy loss of part 1}}\label{proof proposition part1}
\begin{proof}
According to the observation in section~\ref{CGS-EM}, we need to compute an upper bound of the sensitivity $\Delta u=\max_k{\{\Delta \log{p_k}\}}$ of the utility function $u=\log{p}$.

As shown in Equation~(\ref{Equ:sampling equation}), the sampling probability on $w_i=t\in D$ for any given topic $k$ can be computed by
\begin{equation}\label{Equ:sampling equation on D}
p_k \propto r_k=\frac{n_{k}^{t}+\beta}{\sum_{t=1}^{V}(n_{k}^{t}+\beta)}\cdot\frac{n_{m }^{k}+\alpha}{\sum_{k=1}^{K}(n_{m}^{k}+\alpha)}
\end{equation}
And correspondingly, the sampling probability on $w'_i=t'\in D'$ can be computed by
\begin{equation}\label{Equ:sampling equation on D'}
p'_k \propto r'_k=\frac{n_{k}^{t'}+\beta}{\sum_{t=1}^{V}(n_{k}^{t}+\beta)}\cdot\frac{n_{m }^{k}+\alpha}{\sum_{k=1}^{K}(n_{m}^{k}+\alpha)}
\end{equation}
Notably, here we assume $n_m^k|D=n_m^k|D'$ since if the \textit{word replacement} causes any change on $n_m^k$, which also means some sampling process has been changed, then that change would be observed by the adversary.

Then taking into account the normalization constants, $\Delta \log{p_k}$ can be bounded by
\begin{equation}\label{sensitivity function of part1}
\begin{aligned}
\Delta \log{p_k}&=|\log\frac{p'_k}{p_k}|=|\log{(\frac{r_k}{r'_k}\cdot\frac{\sum_k{r'_k}}{\sum_k{r_k}})}|
\\&=|\log{(\frac{r_k}{r'_k})}+\log{(\frac{\sum_k{r'_k}}{\sum_k{r_k}})}|
\\&\leq|\log{(\frac{r_k}{r'_k})}|+|\log{(\frac{\sum_k{r'_k}}{\sum_k{r_k}})}|
\\&\leq\Delta \log{r_k}+\max_k{\{|\log{\frac{r'_k}{r_k}|}\}}
\\&\leq 2\max_k{\{\Delta \log{r_k}\}}=2\max_k{\{|\log{\frac{n_{k}^{t'}+\beta}{n_k^t+\beta}}|\}}
\end{aligned}
\end{equation}
Until now, proposition~\ref{proposition:privacy loss of part 1} has been proved.
\end{proof}

\section{Proof of Theorem~\ref{theorem:composition in single iteration}}\label{proofcomposition in single iteration}
\begin{proof}

Let $P(\cdot)$ and $P'(\cdot)$ denote the probability on $D$ and $D'$ respectively, and for convenience of derivation, the conditional probability $Pr[|]$ are simply represented by $Pr[]$,
then
\begin{equation}\label{Equ: chain effect}
\begin{aligned}
 &P'(\mathcal{S}_i(D')=\mathbf{o_{i}}|\mathcal{S}_{i-1}(D')=\mathbf{o_{i-1}})=P'(\mathcal{S}_i(w_r=t)=o_r)\cdot \\&\prod_{h=1}^{n_{t}-1}P'(\mathcal{S}_i(w_h=t)=o_h^t)
\cdot \prod_{k=1}^{n_{t'}}P'(\mathcal{S}_i(w_k=t')=o_i^{t'})\cdot \\&P'(\mathcal{S}_i(\mathbf{w^{-}})=\mathbf{o}^{-})
\\&\leq e^{\epsilon_i^r}P(\mathcal{S}_i(w_r=t')=o_r)\cdot e^{\sum_{h=1}^{n_{t}-1}{\epsilon_i^{t}}} \prod_{h=1}^{n_{t}-1}P(\mathcal{S}_i(w_h=t)=o_i^t)
\\&\cdot e^{\sum_{k=1}^{n_{t'}}{\epsilon_i^{t'}}}\prod_{k=1}^{n_{t'}}P(\mathcal{S}_i(w_k=t')=o_i^{t'})\cdot P(\mathcal{S}_i(\mathbf{w^{-}})=\mathbf{o}^{-})
\\&=e^{\epsilon_i^i+(n_t-1)\epsilon_i^{t}+n_{t'}\epsilon_i^{t'}}P(\mathcal{S}_i(D)=\mathbf{o_{i}}|\mathcal{S}_{i-1}(D)=\mathbf{o_{i-1}})
\end{aligned}
\end{equation}
where $n_t$ and $n_{t'}$ denote the counts of word $t$ and $t'$ in $D$ respectively and $\mathbf{w^{-}}$ denotes all the \textit{Unrelated words} in $D$.

Then, according to Proposition~\ref{proposition:privacy loss of part 1} and Proposition~\ref{proposition:privacy loss of part 2}, the privacy loss in the $i$-th iteration can be bounded by
\begin{equation*}
\begin{aligned}
\epsilon_i\leq 2\max_k{\{|\log{\frac{n_{k}^{t'}+\beta}{n_k^t+\beta}}|}+(n_{t'}+n_{t}-1)\cdot\log{(1+\frac{1}{\beta})}\}
\end{aligned}
\end{equation*}
Due to the arbitrariness of $t'$, $\epsilon_i$ can be further bounded by
\begin{equation*}
\begin{aligned}
\epsilon_i&\leq 2\{\log{(\frac{\max_{k,t'}{\{n_{k}^{t'}\}}}{\beta}+1)}
\\&+(\max_{t'}\{n_{t'}\}+n_{t}-1)\cdot\log{(1+\frac{1}{\beta})}\}
\end{aligned}
\end{equation*}

Until now, Theorem~\ref{theorem:composition in single iteration} has been proved.
\end{proof}

\section{Proof of Theorem~\ref{theorem:composition in total CGS}}\label{proof composition in total CGS}
\begin{proof}
Given the replaced word $w_r\in D$, the privacy loss of $w_r$ in the $i$-th iteration $\epsilon^r_i$ can be bounded by Theorem~\ref{theorem:composition in single iteration}. Now consider the privacy loss in multiple iterations. Let $\mathcal{S}$ denotes the CGS algorithm throughout the overall training process, $\mathcal{S}_i$ the algorithm in the $i$-th iteration.
\begin{equation*}
\begin{aligned}
&P'(\mathcal{S}(D')=\mathbf{o})
\\&=P'(\mathcal{S}_1(D')=\mathbf{o_1}...,\mathcal{S}_i(D')=\mathbf{o_i},...\mathcal{S}_n(D')=\mathbf{o_n})
\\&=\prod_{i=1}^n {P'(\mathcal{S}_i(D')=\mathbf{o_i}|\mathcal{S}_1(D')=\mathbf{o_1},...,\mathcal{S}_{i-1}(D')=\mathbf{o_{i-1}})}
\\&\leq \prod_{i=1}^n {e^{\epsilon_i}\cdot P(\mathcal{S}_i(D)=\mathbf{o_i}|\mathcal{S}_1(D)=\mathbf{o_1},...,\mathcal{S}_{i-1}(D)=\mathbf{o_{i-1}})}
\\&= e^{\sum_{i=1}^n{\epsilon_i}}P(\mathcal{S}_1(D)=\mathbf{o_1}...,\mathcal{S}_i(D)=\mathbf{o_i},...)
\\&= e^{\sum_{i=1}^n{\epsilon_i}}P(\mathcal{S}(D)=\mathbf{o})
\end{aligned}
\end{equation*}
Until now, we have found the privacy loss $\epsilon_{w_r}$ of $w_r$ over $n$ iterations.

However, due to the arbitrariness of $w_r$, the upper bound of the privacy loss incurred by CGS algorithm should be found by traversing all the words in $D$. Therefore, the total privacy loss $\epsilon$ can be bounded by
\begin{equation*}
\begin{aligned}
\epsilon \leq \max_{w\in D}{\{\sum_{j=1}^n{\epsilon_i^w}\}}
\end{aligned}
\end{equation*}
So far, Theorem~\ref{theorem:composition in total CGS} has been proved.
\end{proof}

\section{Proof of Theorem~\ref{chain effect mitigation theorem}}\label{proof chain effect mitigation theorem}
\begin{proof}

In each iteration of Algorithm~\ref{alg: Hybrid privacy preserving Algorithm}, the privacy loss should be divided into 2 parts: 1). privacy loss $\epsilon_L$ when protecting the word counts information $\mathbf{s}=\{n_k^t\}$, let $\mathcal{A}_1(D)$ be the protection process and $\mathbf{s'}$ the sanitized word counts; 2). privacy loss $\epsilon_I$ in the complete topic sampling process in this iteration, let $\mathcal{A}_2(s',D)$ be the sampling process and $\mathbf{o}$ the outputs.

First consider the first part. Suppose $D'$ is constructed by replacing $w_r=t\in D$ by $t'$, then this replacement will cause $n_k^t-1$ and $n_k^{t'}+1$ on $D'$ for some $k$, and the added noise will need to cover the two caused changes:
\begin{equation}\label{Equ:1st part in the proof of thm3}
\begin{aligned}
&Pr[\mathcal{A}_1(D)=\mathbf{s'})]
\\&=\sum_{\mathbf{s'}}Pr[...,(n_k^t)+\eta_k^t=s_k^t,...,(n_k^{t'})+\eta_k^{t'}=s_k^{t'},...|D]
\\&=\sum_{\mathbf{s'}}\prod_{h,j} Pr[(n_h^j)+\eta_h^j=s_h^j|D]
\\&\leq\sum_{\mathbf{s'}}\prod_{(h,j)\neq(k,t),(k,t')} Pr[(n_h^j)+\eta_h^j=s_h^j|D']\\&\cdot e^{\frac{\epsilon_L}{2}} Pr[(n_k^t-1)+\eta_k^t=s_k^t|D']\cdot e^{\frac{\epsilon_L}{2}} Pr[(n_k^{t'}+1)+\eta_k^{t'}=s_k^{t'}|D']
\\&\leq\sum_{\mathbf{s'}}e^{\epsilon_L}Pr[...,(n_k^t-1)+\eta_k^t=s_k^t,...,(n_k^{t'}+1)+\eta_k^{t'}=s_k^{t'},...|D']
\\&= e^{\epsilon_L}Pr[\mathcal{A}_1(D')=\mathbf{s'})]
\end{aligned}
\end{equation}
Then consider the second part. According to Theorem~\ref{theorem:composition in single iteration}, we focus on the privacy loss on the \textit{Related words} and the \textit{Replace word}. The privacy loss on the \textit{Related words} can be computed by
\begin{equation}\label{Equ:chain effect mitigation in proof}
\begin{aligned}
\epsilon&\leq\max_k\{\Delta\log{p_k}\}=\max_k\{\log\frac{p_k(w=t(t')|D)}{p_k(w=t(t')|D')}\}
\\&=\max_k\{\log\frac{\frac{(n_{k}^{t(t')})^{temp}+\beta}{\sum_{t=1}^{V}((n_k^t)'+\beta)}\cdot\frac{n_{m }^{k}+\alpha}{\sum_{k=1}^{K}(n_{m}^{k}+\alpha)}|D}{\frac{(n_{k}^{t(t')})^{temp}+\beta}{\sum_{t=1}^{V}((n_k^t)'+\beta)}\cdot\frac{n_{m }^{k}+\alpha}{\sum_{k=1}^{K}(n_{m}^{k}+\alpha)}|D'}\}
\\&=\max_k\{\log\frac{(n_{k}^{t(t')})^{temp}+\beta|D}{(n_{k}^{t(t')})^{temp}+\beta|D'}\}=0
\end{aligned}
\end{equation}
where $t(t')$ means $t$ or $t'$ and the last step is because $n_{k}^{t(t')}$ has been privatized by adding Laplace noise. Therefore, according to Theorem~\ref{theorem:composition in single iteration}, the privacy loss $\epsilon_I$ should be
\begin{equation}\label{Equ:2nd part in the proof of thm3}
\begin{aligned}
\epsilon_I\leq 2\cdot\log{(\frac{\max_{k,t'}{\{n_{k}^{t'}\}}}{\beta}+1)}\leq 2\cdot\log{(\frac{C}{\beta}+1)}
\end{aligned}
\end{equation}
where $C$ denotes the clipping bound on $n_k^t$.
Now, combine Equation~(\ref{Equ:1st part in the proof of thm3}) and Equation~(\ref{Equ:2nd part in the proof of thm3}),

\begin{equation*}
\begin{aligned}
&Pr[\mathcal{A}_2(\mathcal{A}_1(D),D)=\mathbf{o}]
\\&=\sum_{\mathbf{s'}}{Pr[\mathcal{A}_1(D)=\mathbf{s'}]\cdot Pr[\mathcal{A}_2(\mathbf{s'},D)=\mathbf{o}]}
\\&\leq\sum_{\mathbf{s'}}e^{\epsilon_L}{Pr[\mathcal{A}_1(D')=\mathbf{s'}]\cdot e^{\epsilon_I}Pr[\mathcal{A}_2(\mathbf{s'},D')=\mathbf{o}]}
\\&\leq e^{\epsilon_L+\epsilon_I}Pr[\mathcal{A}_2(\mathcal{A}_1(D'),D')=\mathbf{o}]
\end{aligned}
\end{equation*}
So far, Theorem~\ref{chain effect mitigation theorem} has been proved.
\end{proof}

\clearpage
\end{spacing}
\end{document}